\pdfoutput=1
\PassOptionsToPackage{table}{xcolor}
\documentclass[11pt]{article}

\usepackage[final]{acl}
\usepackage{hyperref}
\usepackage{times}
\usepackage{latexsym}
\usepackage[T1]{fontenc}
\usepackage[utf8]{inputenc}

\usepackage{microtype}

\usepackage{inconsolata}

\usepackage{graphicx}

\usepackage{multirow}
\def\CHK#1 {\textcolor{magenta}{{\bf [CHK:}~#1{\bf ]}}~}
\def\ADD#1 {\textcolor{cyan}{{\bf [ADD:}~#1{\bf}]}~}

\usepackage{amssymb} % enable /mathbb
\usepackage{amsmath, empheq} % enable /mathbb
\usepackage{amsthm}
\usepackage{tcolorbox}
\usepackage{multicol}
\newtheorem{theorem}{Theorem}[section]

\newtheorem{lemma}[theorem]{Lemma}

\usepackage{booktabs}

\def\AdaReTaKe{AdaR\scalebox{0.8}{E}T\scalebox{0.8}{A}K\scalebox{0.8}{E}}

\title{{\AdaReTaKe}: Adaptive Redundancy Reduction to Perceive Longer for Video-language Understanding}

\author{
Xiao Wang\textnormal{\textsuperscript{1}\footnotemark[1]\footnotemark[2]}  \quad
Qingyi Si\textnormal{\textsuperscript{2}\footnotemark[1]} \quad
Jianlong Wu\textnormal{\textsuperscript{1}\footnotemark[3]}   \quad
Shiyu Zhu\textnormal{\textsuperscript{3}} \quad
Li Cao\textnormal{\textsuperscript{2}} \quad
Liqiang Nie\textnormal{\textsuperscript{1}\footnotemark[3]} \\
\textsuperscript{1}Harbin Institute of Technology, Shenzhen \\
\textsuperscript{2}Huawei Technologies Co., Ltd.
\textsuperscript{3}Shandong University \\
{\tt\small scz.wangxiao@gmail.com, siqingyi@huawei.com, wujianlong@hit.edu.cn} \\
{\tt\small xyzcaoli@outlook.com, nieliqiang@gmail.com}
}

\begin{document}
\maketitle

\renewcommand{\thefootnote}{\fnsymbol{footnote}}  % 让脚注符号变为符号模式
% \footnotetext[2]{Equal contribution. *Work done during an internship at Huawei.} 

\footnotetext[1]{Equal contribution.}  
\footnotetext[2]{Work done during an internship at Huawei.}
\footnotetext[3]{Corresponding author.}

\begin{abstract}

Multimodal Large Language Models (MLLMs) have revolutionized video understanding, yet are still limited by context length when processing long videos. 
Recent methods compress videos by leveraging visual redundancy uniformly, yielding promising results. 
Nevertheless, our quantitative analysis shows that redundancy varies significantly across time and model layers, necessitating a more flexible compression strategy. 
We propose \textbf{{\AdaReTaKe}}, a training-free method that flexibly reduces visual redundancy by allocating compression ratios among time and layers with theoretical guarantees. 
% Integrated into state-of-the-art MLLMs, {\AdaReTaKe} improves
 {\AdaReTaKe}  can be seamlessly integrated into existing MLLMs as a plug-and-play solution, extending 
 their processing capacity from 256 to 2048 frames while preserving critical information. Experiments on VideoMME, MLVU, LongVideoBench, and LVBench datasets demonstrate that {\AdaReTaKe} outperforms existing methods by 2.3\% and 2.8\% for 7B and 72B models, respectively, with even greater improvements of 5.9\% and 6.0\% on the longest LVBench.
Our code is available at \url{https://github.com/SCZwangxiao/video-FlexReduc.git}.
% Our work highlights the significant improvements in long video understanding achievable through more effective redundancy reduction.

\end{abstract}

\section{Introduction}

In pursuit of general intelligence, Multimodal Large Language Models (MLLMs) \cite{2023videochat, bin_videollava_2024, wang_vid_dataflywheel_2024, wang2025haic} have revolutionized video understanding. However, current MLLMs require hundreds of tokens to represent a single image \cite{wang_qwen2vl_2024, li_llava-onevision_2024, wang2023rtq}, limiting video lengths to less than 10 minutes \cite{shen_longvu_2024, gan2023temporal}.

Efforts to extend MLLMs' capabilities for long videos include: agent systems \cite{zhang_llovi_2024} which retrieve and interpret pre-segmented videos but remain constrained by single-model abilities.
Techniques like length extrapolation \cite{zhang_longva_2024} and multi-modal sequence parallelism \cite{xue_longvila_2024} enhance usable video context length but introduce more visual redundancy.
Rather than extending context length, compression-based methods reduce video tokens into shorter sequences by leveraging visual redundancy \cite{bolya_tome_2022}. Many approaches \cite{he_ma-lmm_2024, fei_video-ccam_2024} train Q-Former \cite{li_blip-2_2023} to condense videos guided by language or learnable query tokens. 
% These methods reduce redundancy through prior knowledge in Q-Former, but they are suboptimal since Q-Former, trained from scratch, lacks the extensive world knowledge embedded in LLMs.
Recent advancements \cite{shen_longvu_2024, xiao_retake_2024} integrate compression into MLLM prefilling, yielding promising results. 

% Nevertheless, they apply a fixed ratio throughout compression, failing to address the dynamic nature of visual redundancy, highlighting the need for a more flexible video compression strategy.

\begin{figure}[t]
    \centering
    \includegraphics[width=\linewidth]{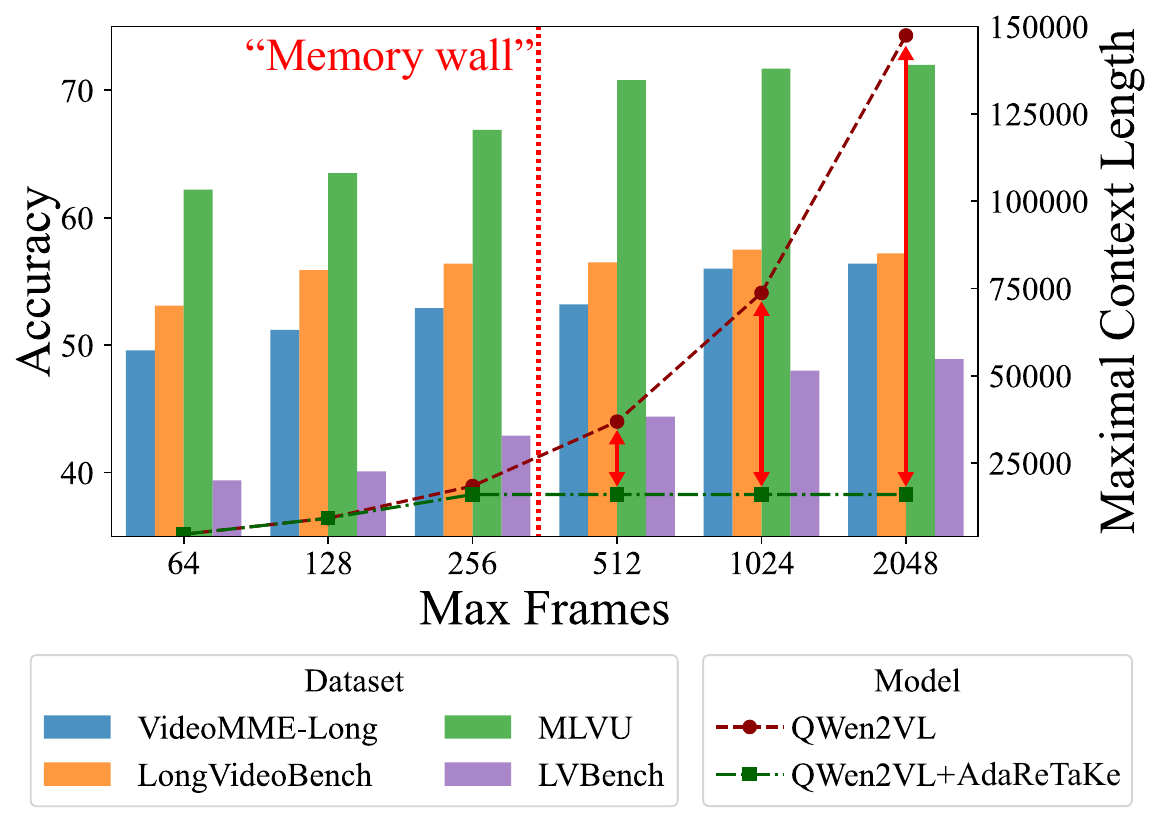}
    \caption{{\AdaReTaKe} enables MLLM to perceive longer with fixed context length for video-language understanding.}
    \label{fig:frame_scaling}
\end{figure}

In this work, we push the boundaries of compression-based methods in two key ways: first, by optimizing the compression algorithm with insights from quantitative analysis; and second, by scaling the number of frames processed to capture more information from the video.

To dive deeper into compression-based methods, we quantitatively analyze visual redundancy by examining the distribution of influential tokens (more likely to be preserved during compression) during MLLM inference, revealing significant variations across video timestamps and LLM layers. 
These findings show that previous methods with fixed compression ratios fail to capture the dynamic nature of visual redundancy, underscoring the need for a more flexible compression strategy.
In light of this, we propose \textbf{{\AdaReTaKe}}, a training-free adaptive video redundancy reduction method.
It features two modules: temporal-adaptive allocation, which adjusts compression ratios for video sequence features over time, and layer-adaptive allocation, which manages KV cache states across layers.
For temporal allocation, we divide a long video into chunks and allocate a compression ratio for each chunk based on the similarity scores between adjacent frames.
For layer allocation, we adjust compression ratios across layers based on video-prompt attention scores. Our theoretical analysis demonstrates that this approach reduces the upper bound of $L_1$ compression loss compared to uniform allocation. 
The combination of the above allocation determines a specific compression ratio for each chunk in each LLM layer.
Finally, we apply chunked prefilling for all chunks and the prompt.  During this process, the KV caches of each chunk are compressed iteratively based on the accumulated attention scores.
{\AdaReTaKe} compresses long videos into shorter sequences, allowing to perceive more informative frames within a fixed GPU memory budget, thereby enhancing long video understanding.
%At each iteration, we compress the KV caches of the current segment according to the accumulated attention scores.

We integrate {\AdaReTaKe} into cutting-edge MLLMs like QWen2-VL \cite{wang_qwen2vl_2024} and LLaVA-Video \cite{zhang_llava-video_2024}, and conduct extensive experiments across various video understanding benchmarks, including VideoMME \cite{fu_videomme_2024}, MLVU \cite{zhou_mlvu_2024}, LongVideoBench \cite{wu_longvideobench_2024}, and LVBench \cite{wang_lvbench_2024}. 
The results show that {\AdaReTaKe} significantly outperforms existing methods, achieving an average improvement of 2.3\% and 2.8\% across datasets for 7B and 72B models, respectively. 
On LVBench, the dataset with the longest average video length, the gains are even more pronounced, with improvements of 5.9\% and 6.0\%, respectively.
Additionally, the results on needle QA and temporal grounding tasks further demonstrate that our approach effectively preserves fine-grained temporal grounding capabilities.
The ablation study validates the effectiveness of our temporal and layer-adaptive budget allocation methods. Through comparison with other compression approaches, it further demonstrates the superiority of our method.
In summary, our contributions are threefold:
\begin{itemize}
    \item We identify uneven visual redundancy across time and MLLM layers and develop {\AdaReTaKe} to adaptively reduce it, expanding MLLM capacity from 256 to 2048 frames for long video understanding.
    \item We design temporal- and layer-adaptive allocation modules to allocate compression ratios across time and MLLM layers, respectively. Theoretical analysis demonstrates that the layer-wise allocation effectively minimizes the upper bound of compression loss.
    \item Our approach achieves state-of-the-art performance, surpassing existing MLLMs by an average of 2.3\% and 2.8\% across 4 datasets for 7B and 72B models, respectively. 
    % Our code is available\footnote{\url{https://anonymous.4open.science/r/2025arrfeb1384-652D}}.
\end{itemize}

\section{Related Work}

\subsection{MLLM for Long Videos}

Most existing multi-modal large language models struggle with extreme token lengths when applied directly to long videos. A commonly used and computationally manageable context length for multimodal training is 8k \cite{shen_longvu_2024}, which restricts video processing to a few minutes.

Early attempts developed \textit{video agent systems} \cite{zhang_llovi_2024, wang_videoagent_2024, luo_video_rag_2024, liu2018attentive} that segment videos into shorter clips and use MLLMs with open-source tools for retrieval, aggregation, and interpretation. However, a single model's capabilities remain limited, reducing overall effectiveness. \textit{Length extrapolation methods} \cite{zhang_longva_2024, shang_intp_2024, wei_visyarn_2024} extend context windows beyond training lengths, but GPU memory still limits context size. To address this, \citeauthor{xue_longvila_2024} introduced LongVILA, a \textit{multi-modal sequence parallelism system} that distributes computation across GPUs, but this adds communication overhead \cite{li_sp_2023}, affecting efficiency. In contrast, \textit{compression-based methods} condense video tokens into shorter sequences. Approaches \cite{he_ma-lmm_2024, fei_video-ccam_2024, cheng_he_mllm_2024, zeng_timesuite_2024, man_adacm2_2024, han_dynfocus_2024} use Q-Former \cite{li_blip-2_2023} for token compression, reducing redundancy by leveraging language or query tokens. However, Q-Former, trained from scratch, lacks the world knowledge embedded in LLMs, making these methods suboptimal. Recent advances \cite{shu_video-xl_2024, shen_longvu_2024, liu_mustdrop_2024, xiao_retake_2024} integrate compression within the LLM, achieving promising results.

\subsection{Token Compression for MLLMs}

% pruning: FastV, FitPrune, ZipVL, FocusLLaVA
% token merging: LOOK-M, SparseVLM
% layer: PyramidDrop, VL-Cache

% Token compression methods for LLMs \cite{xiao_streamingllm_2024, zhang_h2o_2023, feng_ada-kv_2024} have been developed to reduce sequence length with performance loss by evicting less important tokens. 
% Given the higher redundancy in visual tokens compared to language tokens \cite{bolya_tome_2022}, these approaches have been applied to MLLMs as well \cite{chen_fastv_2024, ye_fitprune_2024, he_zipvl_2024, zhu_focusllava_2024}.
% Further advancements include merging evicted tokens to minimize information loss \cite{wan_look-m_2024, zhang_sparsevlm_2024}, and exploring redundancy patterns across LLM layers \cite{xing_pyramiddrop_2024, tu_vl-cache_2024}. However, unlike our adaptive allocation methods, these approaches do not exploit temporal redundancy and allocate compression ratios across layers either monotonically \cite{xing_pyramiddrop_2024} or via heuristic functions \cite{tu_vl-cache_2024}, leading to suboptimal results.

Token compression methods for LLMs \cite{xiao_streamingllm_2024, zhang_h2o_2023, feng_ada-kv_2024} reduce sequence length by evicting less important tokens, typically with some performance loss. Given the higher redundancy in visual tokens compared to language tokens \cite{bolya_tome_2022}, these methods have been extended to MLLMs \cite{chen_fastv_2024, ye_fitprune_2024, he_zipvl_2024, zhu_focusllava_2024}. Advancements include merging evicted tokens to reduce information loss \cite{wan_look-m_2024, zhang_sparsevlm_2024} and analyzing redundancy across layers \cite{xing_pyramiddrop_2024, tu_vl-cache_2024}. However, unlike our adaptive allocation approach, these methods fail to exploit temporal redundancy and allocate compression ratios either monotonically \cite{xing_pyramiddrop_2024} or via heuristics \cite{tu_vl-cache_2024}, resulting in suboptimal performance.

In this paper, we advance token compression methods for MLLMs by adaptively adjusting the compression ratio across timestamps and layers to reduce redundancy more effectively.

\section{Preliminary Analysis} \label{sec:analysis}

In this section, we provide a quantitative analysis of the visual redundancy with MLLM for long video understanding. 
Intuitively, redundancy varies across dimensions: at the frame level, static scenes are more redundant than dynamic ones, and at the model level, deeper layers focus on more abstract features, leading to different attention patterns. 
To quantify this, we measure redundancy through the ratio of heavy-hitters \cite{zhang_h2o_2023}, a set of influential tokens essential for generation. By identifying these across dimensions, we validate the varying redundancy levels, providing a strong motivation for our approach to achieve more flexible and efficient compression.
% Our findings reveal that redundancy varies across different timestamps and LLM layers, highlighting the need for a novel video compression algorithm that can adaptively adjust the compression ratio across these dimensions.

\begin{figure}[t]
    \centering
    \includegraphics[width=\linewidth]{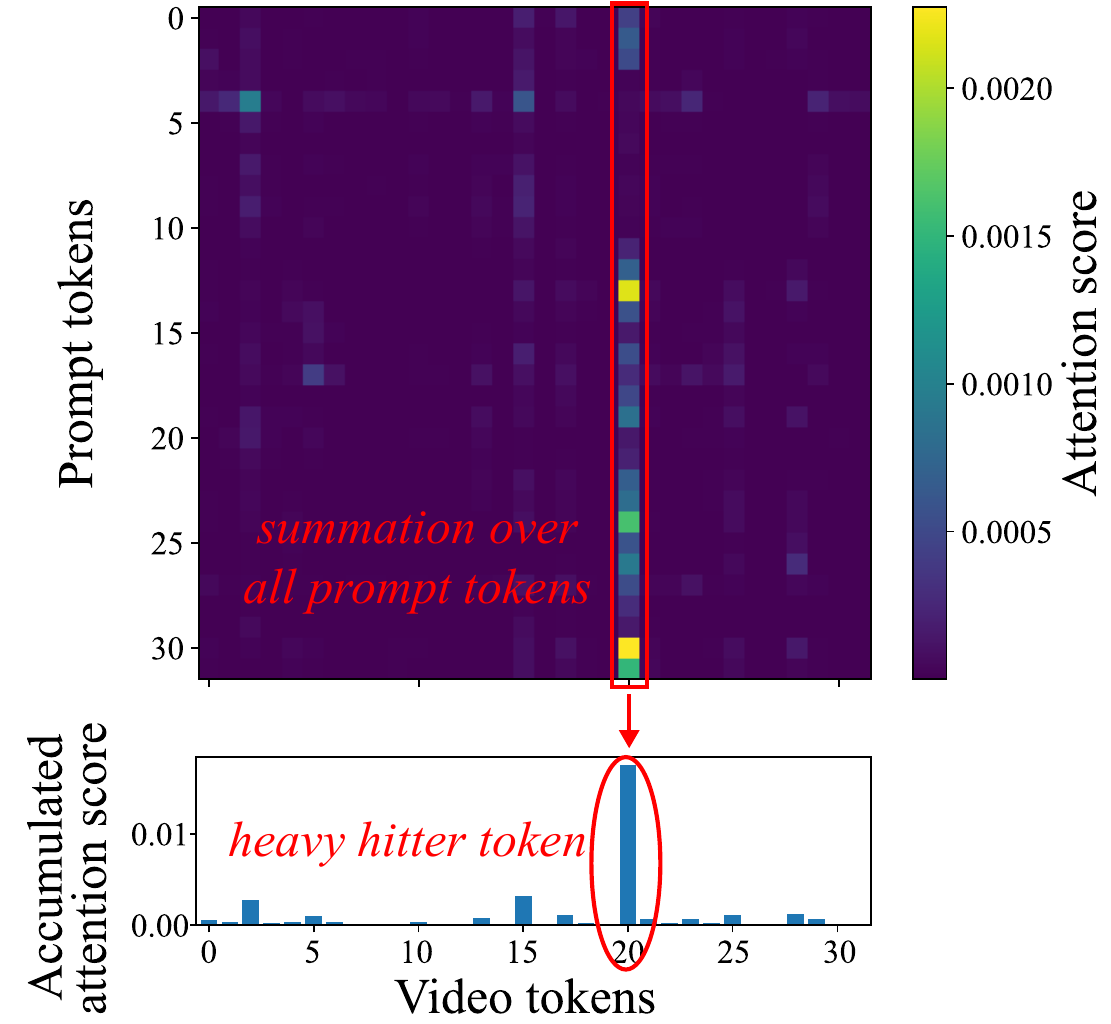}
    \caption{Illustrating example of a heavy hitter. We adopt the heavy hitter ratio to measure the redundancy}
    \label{fig:hitter_definition}
\end{figure}

\noindent\textbf{Heavy-hitter ratio to measure redundancy.}
Denote the number of attention heads as $h$, length of prompt and video tokens as $L_t$ and $L_v$, respectively, and the attention scores of them in layer $l$ is $\textbf{A}^{(l)}\in\mathbb{R}^{h\times L_t \times L_v}$. We first calculate the prompt-accumulated head-average attention scores to measure the influence of each video token during generation $\mathbf{a}\in\mathbb{R}^{L_v}$:
\begin{equation}\label{eq:head_average}
    \mathbf{a}^{(l)} = \sum_{j=1}^{L_t}\frac{1}{h}\sum_{i=1}^{h}{\textbf{A}^{(l)}[i,j]}.
\end{equation}
We then calculate the \textit{heavy-hitter ratio} $\lambda^{(l)}\in\mathbb{R}$:
\begin{equation}
    \lambda^{(l)}=\frac{1}{L_v}\sum_{i=1}^{L_v}{\mathbb{1}{\left(\mathbf{a}^{(l)}[i] > p \max{\left\{ \mathbf{a}^{(l)} \right\}} \right)}},
\end{equation}
where $\mathbb{1}(\cdot)\in\{0,1\}$ is the indicator function and $p=0.01$ is a heuristic constant.A video token is considered important (called a \textit{heavy-hitter}) if its accumulated attention $\mathbf{a}^{(l)}[i]$ exceeds $p$ times the maximum attention value in $\mathbf{a}^{(l)}$.

\begin{figure}[t]
    \centering
    \includegraphics[width=\linewidth]{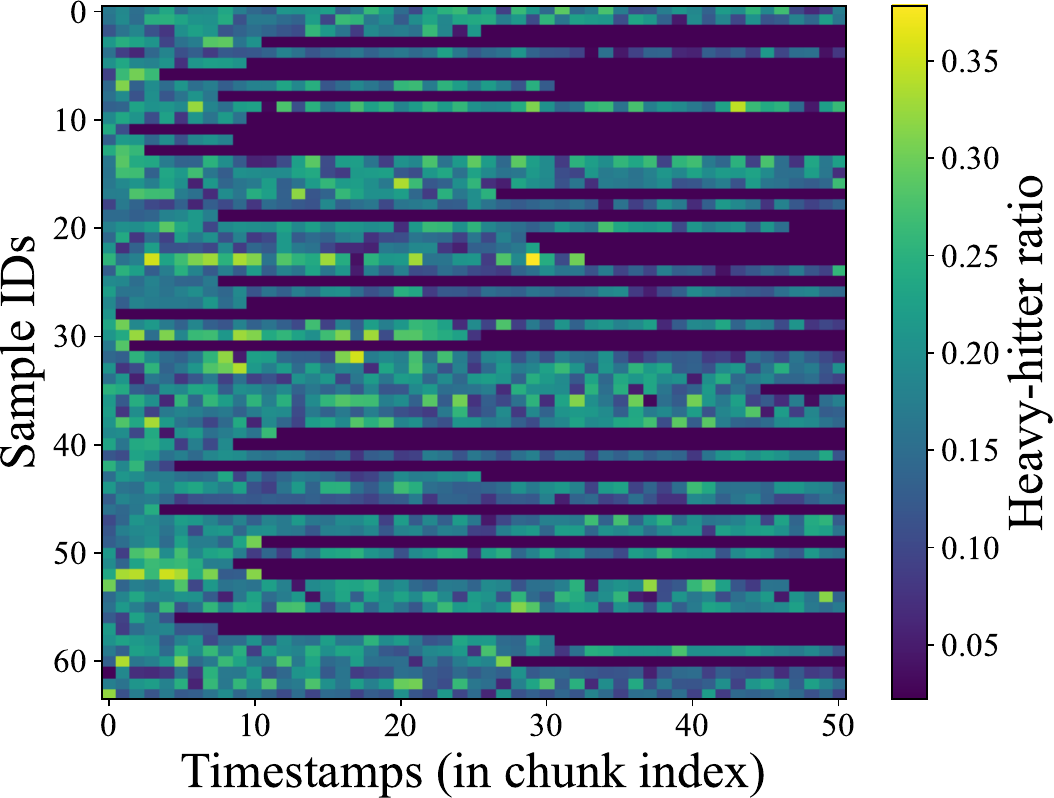}
    \caption{Heavy-hitter ratio among timestamps, showing the unevenly distributed temporal redundancy.
    The horizontal shaded bars indicate timestamps where the video has ended.}
    \label{fig:hitter_temporal}
\end{figure}

\begin{figure}[t]
    \centering
    \includegraphics[width=\linewidth]{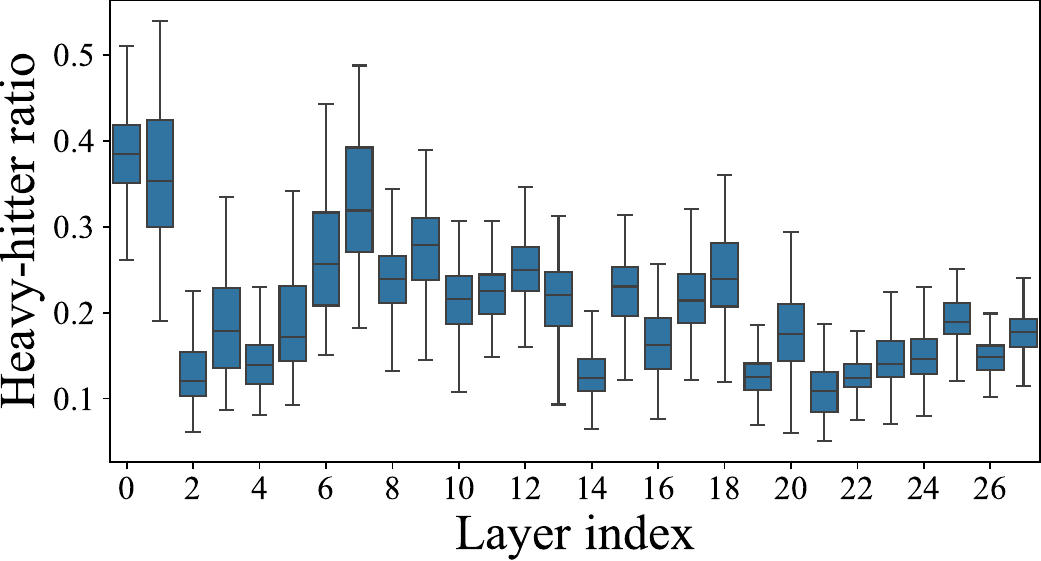}
    \caption{Heavy-hitter ratio among layers, showing the unevenly distributed redundancy among LLM layers.}
    \label{fig:hitter_layer}
\end{figure}

\noindent\textbf{Redundancy among video timestamps.} 
To explore the distribution of redundancy over time, we first split the video tokens into chunks of 10 seconds, and denote the heavy hitter ratio chunk $t$ as $\lambda^{(t,l)}$. 
We randomly sampled 64 videos from VideoMME \cite{fu_videomme_2024} and plotted the layer-averaged heavy hitter ratio $\sum_k\lambda^{(t,l)}$ across different chunks as a heatmap in \autoref{fig:hitter_temporal}. 
The temporal redundancy is unevenly distributed, with the heavy-hitter ratio varying up to 3x within a video, as highlighted by the red circle in \autoref{fig:hitter_definition}.

\noindent\textbf{Redundancy among LLM layers.}
To investigate the distribution of redundancy across LLM layers in MLLM, we utilized all videos from VideoMME \cite{fu_videomme_2024} and plotted heavy hitter ratio $\sum_k\lambda^{(t,l)}$ across different layers as a boxplot in \autoref{fig:hitter_layer}. 
The redundancy is unevenly distributed among the LLM layers. Generally, the heavy hitter ratio is lower in the deeper layers, but significant fluctuations are observed, with local minima at layers 2, 14, and 21, and maxima at layers 7 and 18.
This indicates that token compression methods that monotonically assign higher compression ratios to deeper layers, such as PyramidDrop \cite{xing_pyramiddrop_2024}, are suboptimal for video understanding.

To maximize the use of informative frames within a fixed GPU memory budget, we must design a video compression algorithm that adaptively adjusts the compression ratio across different timestamps and LLM layers.

\section{Methods}

\begin{figure}[t]
    \centering
    \includegraphics[width=\linewidth]{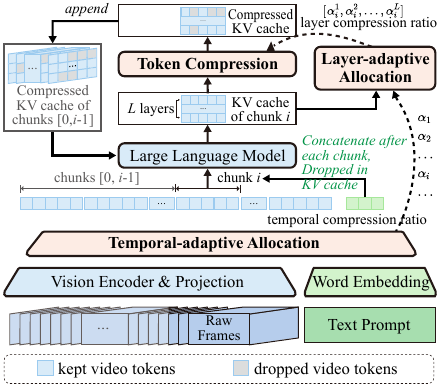}
    % \vspace{-1em}
    \caption{    
    Illustration of {\AdaReTaKe}.
    }
    \label{fig:method_overview}
    % \vspace{-1em}
\end{figure}

\subsection{Overview}

The architecture of {\AdaReTaKe} is shown in \autoref{fig:method_overview}. 
To flexibly reduce redundancy across timestamps, we divide video sequences into equal chunks and the \textbf{Temporal-adaptive Allocation} module dynamically applies distinct compression ratios to each chunk. 
For redundancy across layers, the \textbf{Layer-adaptive Allocation} module assigns varying compression ratios to LLM layers.
Finally, the \textbf{Token Compression} module compresses the KV cache after each chunk's prefilling based on the compression ratios determined by the previous modules, reducing the video sequence length in an MLLM.
The general pipeline and these three modules are detailed below.

\subsection{General Pipeline}

% To reduce the redundancy of video tokens, {\AdaReTaKe} leverages the autoregressive nature of LLM to split video sequences into equal chunks and apply chunked prefilling for each chunk. In each iteration, text prompts are concatenated after video tokens to guide token compression using the video-prompt attention scores.

Denote $T$ number of frames, $N$ number of tokens in each frame, $\tau$ number of frames in a chunk (can divide $T$), $S$ prompt length, $L$ number of LLM layers, and $C_{max}$ is a refined maximal context length. 

Given raw frames and a text prompt as input, the visual encoder and projection layer derive video features $\mathbf{M} \in \mathbb{R}^{T\times N \times d}$, and the word embedding layer derives prompt features $\mathbf{P} \in \mathbb{R}^{S \times d}$. We split visual features into chunks of $\tau$ frames:
\begin{equation}
    \mathcal{M}=\left[ 
        \mathbf{M}_1, \mathbf{M}_2, ..., \mathbf{M}_{T/\tau} 
    \right], \mathbf{M}_i\in \mathbb{R}^{\tau \times N \times d}.
\end{equation}
The temporal-adaptive allocation module will produce a compression ratio (length after compression/original length) for each chunk based on the number of tokens in $\mathcal{M}$ and $C_{max}$:
% \begin{empheq}[left=\empheqlbrace]{align}\label{eq:temp_alloc_overall}
%     & \left[ 
%         \alpha_1, \alpha_2, ..., \alpha_{T/\tau} 
%     \right],\alpha_i\in\mathbb{R}, \\
%     & \alpha_1+\alpha_2+ ...+ \alpha_{T/\tau}=\frac{C_{max}-L}{TN}.
% \end{empheq}
\begin{align}\label{eq:temp_alloc_overall}
    &\left[ \alpha_1, \alpha_2, \dots, \alpha_{T/\tau} \right], \quad \alpha_i \in \mathbb{R}, \\
    &\text{s.t.} \quad \alpha_1 + \alpha_2 + \dots + \alpha_{T/\tau} = \frac{C_{\text{max}} - S}{TN}.
\end{align}
The above equation ensures the final total sequence length (in KV cache memory) is $C_{max}$. Note that we do not consider memory usage for other operations since for long sequence inference the KV cache occupies the most GPU memory \cite{coleman_kvquant_2024}.

We employ chunk-based processing instead of single-frame processing to enhance the robustness of the allocation process and reduce memory overhead in temporal-adaptive allocation, as detailed in Section~\ref{sec:layer_alloc}.%We use chunks instead of a single frame to make this allocation process more robust. Besides, it also reduces the memory overhead in temporal-adaptive allocation, see Section~\ref{sec:layer_alloc}.

Due to the autoregressive nature of LLMs, chunked prefilling is applied to each chunk, which is functionally equivalent to standard prefilling \cite{zeng_chunked_prefill_2024}. During the $i$-th iteration, chunk $i$ is first prefilled. For each layer $l$, the query states of the prompt $\mathbf{Q}^{(l)}_i \in \mathbb{R}^{S \times d}$ and the KV caches of chunk $i$ $\mathbf{K}^{(l)}_i, \mathbf{V}^{(l)}_i \in \mathbb{R}^{h\times \tau N \times d}$ are stored, where $h$ is the number of heads. These, along with the chunk's compression ratio $\alpha_i$, are processed by the layer-adaptive allocation module to determine the compression ratio for each layer:
% \begin{empheq}[left=\empheqlbrace]{align}\label{eq:layer_alloc_overall}
%     &\left[ 
%         \alpha_i^{(1)}, \alpha_i^{(2)}, ..., \alpha_i^{(l)} 
%     \right],\alpha_i^{(l)}\in\mathbb{R}, \\
%     &\frac{\alpha_i^{(1)}+\alpha_i^{(2)}+ ...+ \alpha_i^{(l)}}{L}=\alpha_i.
% \end{empheq}
\begin{align}\label{eq:layer_alloc_overall}
    &\left[ \alpha_i^{(1)}, \alpha_i^{(2)}, \dots, \alpha_i^{(l)} \right], \quad \alpha_i^{(l)} \in \mathbb{R}, \\
    \text{s.t.} \quad &\frac{\alpha_i^{(1)} + \alpha_i^{(2)} + \dots + \alpha_i^{(L)}}{L} = \alpha_i.
\end{align}
Finally, token compression is applied to the visual KV caches of chunk $i$, deriving the compressed KV cache $\hat{\mathbf{K}}^{(l)}_i,\hat{\mathbf{V}}^{(l)}_i\in\mathbb{R}^{\alpha_i^{(l)}\tau N\times d}$. The prompt states are dropped except in the last chunk.

\subsection{Temporal-adaptive Allocation}\label{sec:temporal_alloc}

Given chunked video frames $\mathcal{M}$ and maximal context length $C_{max}$, this module calculates the compression ratio for each chunk.

For video features of the $i$-th chunk $\mathbf{M}_i\in \mathbb{R}^{\tau\times N \times d}$, we first calculate the distance between adjacent frames $\mathbf{d}_i\in\mathbb{R}^{\tau-1}$:
\begin{equation}
    \mathbf{d}_i[t] = 1 - \sum_{j=1}^N
    \frac{\textmd{Sim}\left( \mathbf{M}_i[t,j], \mathbf{M}_i[t+1,j] \right)}
    {N},
\end{equation}
where $\textmd{Sim}(\cdot)$ represents the cosine similarity. 
We then average $\mathbf{d}_i$ among its $\tau-1$ frames to get the averaged distance of $i$-th chunk $\bar{d_i}\in\mathbb{R}$, which reflects the temporal redundancy within the chunk. %. This can represent the temporal redundancy within the chunk.
Finally, the compression ratio $\alpha_i$ for each chunk is computed by allocating the maximal context length $C_{\text{max}}$ proportionally to the mean distances: %Finally, we allocate the maximal context length among chunks by calculating the compression ratio for each chunk $\alpha_i$:
\begin{equation}
    \alpha_i=
    \frac{C_{max}-S}{TN}
    % \bar{d_i}/\sum_{i=1}^{T/\tau}\bar{d_i}
    \cdot \frac{\bar{d_i}}{\sum_{i=1}^{T/\tau}\bar{d_i}}.
\end{equation}

\subsection{Layer-adaptive Allocation}\label{sec:layer_alloc}

When prefilling chunk $i$ in the $l$-th LLM layer, we store the query states of the prompt $\mathbf{Q}^{(l)}_i$, KV cache of chunk $\mathbf{K}^{(l)}_i, \mathbf{V}^{(l)}_i$. This module calculates the compression ratio for chunk $i$ in each layer.

In the $l$-th layer, we first calculate the attention score between prompt and the video tokens $\textbf{A}_i^{(l)}\in\mathbb{R}^{h\times S \times \tau N}$. We then calculate the head-averaged accumulated scores along all prompt tokens to measure the significance score of each token to the prompt, $\mathbf{a}_i^{(l)}\in\mathbb{R}^{\tau N}$:
\begin{equation}\label{eq:sig_score}
    \mathbf{a}_i^{(l)} = \sum_{j=1}^{S}\frac{1}{h}\sum_{i=1}^{h}{\textbf{A}_i^{(l)}[i,j]}.
\end{equation}

To measure the significance of each layer, we calculate the number of tokens with large significance scores, denoted as $s_i^{(l)}\in\mathbb{Z}$:
% \begin{empheq}[left=\empheqlbrace]{align}
%     s_i^{(l)} &= \sum_{j=1}^{\tau N}\mathbb{1}\left( 
%     \mathbf{a}_i^{(l)}[j] > \hat{a}_i\right),\\
%     \hat{a}_i &= K\textmd{thValue}\left(\left[
%     \mathbf{a}_i^{(1)}||\mathbf{a}_i^{(2)}||...||\mathbf{a}_i^{(l)}\right]\right), \\
%     K &= \alpha_i\tau NL,
% \end{empheq}
\begin{align}
    s_i^{(l)} &= \sum_{j=1}^{\tau N}\mathbb{1}\left( \mathbf{a}_i^{(l)}[j] > \hat{a}_i\right), \\
    \text{s.t.}\quad \hat{a}_i &= \mathop{\mathrm{K\text{-}th}}\left(\mathbf{a}_i^{(1)} \Vert \cdots \Vert \mathbf{a}_i^{(l)}\right), \\ 
    K &= \alpha_i\tau NL.
\end{align}
where $\mathbb{1}(\cdot)\in\{0,1\}$ is the indicator function, $\mathop{\mathrm{K\text{-}th}}(\cdot)$ denotes the $K$-th largest value in the vector, and $||$ denotes vector concatenation operation. Finally, we allocate the compression ratio of each layer by re-weighting the total compression ratio of current $\alpha_i$ in each layer:
% \begin{empheq}[left=\empheqlbrace]{align}\label{eq:layer_alloc_detail}
%     \alpha_i^{(l)} &= \frac{w_i^{(l)}}{\sum_{k=1}^{(l)} w_i^{(k)}}\alpha_i L, \\
%     w_i^{(l)} &= \frac{s_i^{(l)}}{\sum_{k=1}^{(l)} s_i^{(k)}}.
% \end{empheq}
\begin{align}\label{eq:layer_alloc_detail}
    \alpha_i^{(l)} &= w_i^{(l)}\alpha_i, \\
    w_i^{(l)} &= \frac{s_i^{(l)}}{\sum_{k=1}^l s_i^{(k)}}.
\end{align}
Note that sometimes the $\hat{w}_i^{(k)}$ above might be too small. To ensure numerical stability, we introduce a minimal weight $\epsilon = 0.01$ and compute the re-normalized re-weighting factor $\hat{w}_i^{(l)}$:%we use the re-normalized re-weighting factor $\hat{w}_i^{(k)}$ to ensure a minimal weight $\epsilon=0.01$:
\begin{equation}
    \hat{w}_i^{(l)} = \frac{\max(w_i^{(l)} - \epsilon, 0)}{\sum_{k=1}^l \max(w_i^{(k)} - \epsilon, 0)} (1 - L \epsilon) + \epsilon.
\end{equation}

For memory-efficient implementation, we calculate Eqn.~(\ref{eq:sig_score}) after each layer.%, eliminating the need to store query states during prefilling.

\subsection{Token Compression}\label{sec:token_compression}

After prefilling the $i$-th chunk, we first drop the prompt tokens in the KV cache (except the last chunk). 
Based on the compression ratio derived from Eqn.~\ref{eq:layer_alloc_detail}, we then compress video tokens by selecting tokens with the top significant scores and then update the KV cache in each layer $\mathbf{K}^{(l)}$, $\mathbf{V}^{(l)}$:
% \begin{empheq}[left=\empheqlbrace]{align}
%     \mathcal{I} &= \text{ArgTopK}(\mathbf{a}_i^{(l)}), K=\alpha_i^{(l)} \tau N, \\
%     % \mathbf{K}^{(l)} &\gets  \textmd{Concatenation}\left( \mathbf{K}^{(l)} || \mathbf{K}_i^{(l)}[:,\mathcal{I}] \right ), \\
%     % \mathbf{V}^{(l)} &\gets  \textmd{Concatenation}\left( \mathbf{V}^{(l)} || \mathbf{V}_i^{(l)}[:,\mathcal{I}]\right ),
%     \mathbf{K}^{(l)} &\gets  \left[ \mathbf{K}^{(l)} || \mathbf{K}_i^{(l)}[:,\mathcal{I}] \right ], \\
%     \mathbf{V}^{(l)} &\gets  \left[ \mathbf{V}^{(l)} || \mathbf{V}_i^{(l)}[:,\mathcal{I}]\right ],
% \end{empheq}
\begin{align}
    \mathcal{I} &= \text{ArgTopK}(\mathbf{a}_i^{(l)}), \quad K = \alpha_i^{(l)} \tau N, \\
    \mathbf{K}^{(l)} &\gets \left[ \mathbf{K}^{(l)} \parallel \mathbf{K}_i^{(l)}[:,\mathcal{I}] \right], \\
    \mathbf{V}^{(l)} &\gets \left[ \mathbf{V}^{(l)} \parallel \mathbf{V}_i^{(l)}[:,\mathcal{I}] \right].
\end{align}
where $\text{ArgTopK}(\cdot)$ denotes the indices of $K$ elements with the largest value in the vector.

We also provide a theoretical guarantee for our layer-wise budget allocation method. See \autoref{sec:proof_token_comp} for more details.
\begin{theorem}
Let $I_i^{(l)}\in\{0,1\}$ denotes whether token $i$ in layer $l$ is kept during compression.
Given the token sequence budget $\sum_l\sum_iI_i^{(l)}=K$, making token compression choices $\left\{\mathbf{I}_*^{(l)}\right\}_{l=1}^L$ based on top $K$ values in $\left\{A_i^{(l)}\right\}$ can achieve a near-optimal minimization of the upper bound of token compression loss to $\epsilon_*^{(l)}$:
\begin{equation}
    \epsilon_*^{(l)}
    \le
    2C
    +
    2C
    \left(\frac{\epsilon_{opt}^{(l)}}{2C}-1\right)^
    {1-\frac{1}{e}},
\end{equation}
where $\epsilon_{opt}^{(l)}$ is the theoretical minimal of $\epsilon^{(l)}$ and $C$ is a constant.
\end{theorem}

\section{Experiments}

% \usepackage{multirow}
\begin{table*}[t]

\centering

\resizebox{0.9\textwidth}{!}{
\begin{tabular}{ccccccc}
\hline
{\color[HTML]{363636} } & {\color[HTML]{363636} } & \multicolumn{2}{c}{{\color[HTML]{363636} VideoMME}} & {\color[HTML]{363636} MLVU} & {\color[HTML]{363636} LongVideoBench} & {\color[HTML]{363636} LVBench} \\ \cline{3-7} 
\multirow{-2}{*}{{\color[HTML]{363636} Model}} & \multirow{-2}{*}{{\color[HTML]{363636} LLM Size}} & {\color[HTML]{363636} Long} & {\color[HTML]{363636} Overall} & dev & {\color[HTML]{363636} val} & val \\ \hline
\rowcolor[HTML]{F2F2F2} 
GLM-4V-Plus & - & - & 70.8 & - & - & 58.7 \\
\rowcolor[HTML]{F2F2F2} 
GPT-4o & - & 65.3 & 71.9 & 64.6 & 66.7 & 27.0 \\
\rowcolor[HTML]{F2F2F2} 
Gemini-1.5-Pro & - & 67.4 & 75.0 & - & 64.0 & 33.1 \\ \hline
VITA-1.5 & 7B & 47.1 & 56.1 & - & - & - \\
mPLUG-Owl3 & 7B & 50.1 & {\color[HTML]{363636} 59.3} & 63.7 & 52.1 & - \\
NVILA & 8B & 54.8 & 64.2 & {\color[HTML]{363636} 70.1} & 57.7 & - \\
ByteVideoLLM & 14B & 56.4 & 64.6 & 70.1 & - & - \\
TPO & 7B & 55.4 & 65.6 & {\color[HTML]{363636} 71.1} & 60.1 & - \\
VideoLLaMA3 & 7B & - & 66.2 & {\color[HTML]{363636} 73.0} & 59.8 & 45.3 \\
\rowcolor[HTML]{EAFAF1} 
{\color[HTML]{A5A5A5} LLaVA-Video} & {\color[HTML]{A5A5A5} 7B} & {\color[HTML]{A5A5A5} 52.4} & {\color[HTML]{A5A5A5} 63.3} & {\color[HTML]{A5A5A5} 67.0} & {\color[HTML]{A5A5A5} 58.2} & {\color[HTML]{A5A5A5} 43.1} \\
\rowcolor[HTML]{EAFAF1} 
LLaVA-Video+{\AdaReTaKe} & 7B & 53.9 & 64.0 & 70.6 & 59.6 & 49.6 \\
\rowcolor[HTML]{EAFAF1} 
{\color[HTML]{A5A5A5} Qwen2-VL} & {\color[HTML]{A5A5A5} 7B} & {\color[HTML]{A5A5A5} 53.8} & {\color[HTML]{A5A5A5} 63.3} & {\color[HTML]{A5A5A5} 66.9} & {\color[HTML]{A5A5A5} 55.6} & {\color[HTML]{A5A5A5} 42.4} \\
\rowcolor[HTML]{EAFAF1} 
QWen2-VL+{\AdaReTaKe} & 7B & 56.4 & 64.2 & 72.0 & 57.2 & 48.9 \\
\rowcolor[HTML]{EAFAF1} 
{\color[HTML]{A5A5A5} Qwen2.5-VL} & {\color[HTML]{A5A5A5} 7B} & {\color[HTML]{A5A5A5} 55.6} & {\color[HTML]{A5A5A5} 65.4} & {\color[HTML]{A5A5A5} 70.2} & {\color[HTML]{A5A5A5} 59.5} & {\color[HTML]{A5A5A5} 45.3} \\
\rowcolor[HTML]{EAFAF1} 
QWen2.5-VL+{\AdaReTaKe} & 7B & \textbf{58.3} & \textbf{67.7} & \textbf{75.0} & \textbf{62.6 }& \textbf{51.2} \\ \hline
LLaVA-OneVision & 72B & 60.0 & 66.3 & 68.0 & 61.3 & - \\
Oryx-1.5 & 32B & 59.3 & {\color[HTML]{363636} 67.3} & 72.3 & 62.0 & 30.4 \\
Aria & 8x3.5B & 58.8 & 67.6 & 70.6 & 65.3 & - \\
LLaVA-Video & 72B & 61.5 & {\color[HTML]{363636} 70.6} & {\color[HTML]{363636} 74.4} & 61.9 & - \\
Qwen2-VL & 72B & 62.2 & 71.2 & - & 60.4 & 41.3 \\
InternVL2.5 & 72B & 62.6 & 72.1 & 75.7 & 63.6 & 43.6 \\
\rowcolor[HTML]{EAFAF1} 
{\color[HTML]{A5A5A5} Qwen2.5-VL} & {\color[HTML]{A5A5A5} 72B} & {\color[HTML]{A5A5A5} 63.9} & {\color[HTML]{A5A5A5} 72.6} & {\color[HTML]{A5A5A5} 74.6} & {\color[HTML]{A5A5A5} 65.9} & {\color[HTML]{A5A5A5} 47.3} \\
\rowcolor[HTML]{EAFAF1} 
Qwen2.5-VL+{\AdaReTaKe} & 72B & \textbf{65.0} & \textbf{73.5} & \textbf{78.1} & \textbf{67.0} & \textbf{53.3} \\ \hline
\end{tabular}
}

\caption{
Performance comparison on long video understanding. {\AdaReTaKe} achieves consistent gains when integrated into various MLLMs.
}
\label{tab:comparison_mllm}

\end{table*}

\subsection{Benchmarks and Implementations}

% VideoMME
\textbf{Video-MME.} Video Multi-Modal Evaluation~\cite{fu_videomme_2024} is a pioneering benchmark designed for evaluating video analysis, with diverse video types, and durations. It comprises 900 videos totaling 256 hours, with 2,700 manually labeled complex multiple-choice question-answer pairs across 30 subfields. It has three subsets of different durations: short ($<$ 2min), medium (4min $\sim$ 15min), and long (30min $\sim$ 60min).
\textbf{MLVU.} Multi-task Long Video Understanding Benchmark (MLVU)~\cite{zhou_mlvu_2024} has the widest range of video length ranging from 3 minutes to 2 hours. MLVU includes nine evaluation tasks including topic reasoning, anomaly recognition, video summarization, and plot question-answering. 
% The benchmark uses absolute accuracy to score multiple choice tasks, while Generation tasks are assessed based on alignment with ground truth, with GPT-4 assigning scores from 1 to 10. 
% Performance metrics include average scores for multiple-choice (M-Avg) and generation tasks (G-Avg) across tasks.
% \begin{itemize}
%     \item TR: Topic Reasoning
%     \item AR: Anomaly Recognition
%     \item VS: Video Summarization
%     \item NQA: Needle Question-Answering
%     \item ER: Ego Reasoning
%     \item PQA: Plot Question-Answering
%     \item SSC: Sub-Scene Captioning
%     \item AO: Action Order
%     \item AC: Action Count
% \end{itemize}
% LVbench
\textbf{LongVideoBench}~\cite{wu_longvideobench_2024} is a benchmark for long-context video understanding, consisting of videos up to one hour in length. It includes 3,763 videos with 6,678 annotated multiple-choice questions across 17 categories, focusing on referring reasoning that requires retrieving and analyzing detailed multimodal information from specific temporal segments. 
\textbf{LVBench.} LVBench~\cite{wang_lvbench_2024} is a comprehensive benchmark for long video understanding, with an average video length of 4,101 seconds—4 times longer than VideoMME~\cite{fu_videomme_2024} and 5 times longer than MLVU~\cite{zhou_mlvu_2024}. It includes 1,549 annotated multiple-choice question-answer pairs covering a wide range of tasks, including entity recognition, event understanding, key information retrieval, temporal grounding, and reasoning.

\noindent \textbf{Implementation Details.} We integrated {\AdaReTaKe} into various MLLMs, including LLaVA-Video-7B \cite{zhang_llava-video_2024}, QWen2VL-7B \cite{wang_qwen2vl_2024}, QWen2.5VL-7B, and QWen2.5VL-72B.
We densely sampled the video at 2 frames per second (fps), with a maximum of 2048 and 1024 frames for 7B and 72B models, respectively. 
For our main results (Section~\ref{sec:comp_analsysi}), we chose the maximal context length $C_{max}$ as 16K. 
In the ablation studies (Section\ref{sec:ablation_studies}), we reduced the maximum number of sampled frames to 1024 and the context length to 1K without specification. 
The evaluation is conducted using LMMs-Eval \cite{zhang_2024_lmmseval}.
% All experiments were conducted on 8 Ascend 910B NPUs.

% \subsection{Comparison Analysis}\label{sec:comp_analsysi}
\subsection{Main Results}\label{sec:comp_analsysi}

% \noindent \textbf{Baselines.} Our {\AdaReTaKe} can be integrated into various MLLMs. We incorporated LLaVA-Video-7B \cite{zhang_llava-video_2024}, QWen2VL-7B \cite{wang_qwen2vl_2024}, QWen2.5VL-7B, and QWen2.5VL-72B as our base MLLMs. 

% Results on VideoMME, MLVU, LVBench, LongVideoBench, (Hour1K)
% \usepackage{multirow}
\begin{table}[t]
\resizebox{0.5\textwidth}{!}{
\begin{tabular}{ccccc}
\hline
{\color[HTML]{363636} } & \multicolumn{2}{c}{{\color[HTML]{363636} VideoMME}} & {\color[HTML]{363636} MLVU} & {\color[HTML]{363636} LVBench} \\ \cline{2-5} 
\multirow{-2}{*}{{\color[HTML]{363636} Method}} & {\color[HTML]{363636} Long} & {\color[HTML]{363636} Overall} & val & val \\ \hline
FastV & 53.5 & 61.2 & 63.2 & 42.3 \\
FitPrune & 53.6 & 61.2 & 63.6 & 42.0 \\
LOOK-M & 53.6 & 61.0 & 63.8 & 42.6 \\
SparseVLM & 54.4 & 60.7 & 63.0 & 43.9 \\
PyramidDrop & 53.1 & 60.5 & 63.7 & 41.6 \\
VL-Cache & 53.2 & 61.3 & 64.5 & 42.4 \\ \hline
{\AdaReTaKe} & \textbf{55.1} & \textbf{62.2} & \textbf{65.6} & \textbf{44.8} \\ \hline
\end{tabular}
}

\caption{Comparison with other token compression methods for MLLMs. {\AdaReTaKe} outperforms existing approaches by employing a theoretically grounded budget distribution mechanism, in contrast to heuristic or suboptimal allocation strategies.}
\label{tab:comparison_token_comp}
\end{table}

\begin{table*}[t!]
    \centering

\resizebox{\textwidth}{!}{
\begin{tabular}{lcccccc}
\hline
\multicolumn{1}{c}{{\color[HTML]{363636} Model}} & {\color[HTML]{363636} Max frames} & {\color[HTML]{363636} Context length} & {\color[HTML]{363636} VideoMME-L} & {\color[HTML]{363636} MLVU} & {\color[HTML]{363636} LVBench} & $\Delta_{avg}$ \\ \hline
0 QWen2VL-7B & 128 & 9K & 51.2 & 63.5 & 40.1 & - \\
1 +token compression & 128 & 1K & 50.6 & 62.7 & 39.2 & -0.8 \\
2 +scale up frames & 1024 & 1K & 53.8 & 63.9 & 42.3 & +2.5 \\
3 +layer-wise allocation & 1024 & 1K & 54.3 & 64.6 & 43.5 & +0.8 \\
4 +temporal allocation & 1024 & 1K & 55.1 & 65.6 & 44.8 & +1.0 \\
5 +scale up context length & 1024 & 16K & 56.0 & 71.7 & 48.0 & +3.4 \\
6 +scale up frames & 2048 & 16K & 56.4 & 72.0 & 48.9 & +0.6 \\ \hline
\end{tabular}
}

\caption{Ablation study on different components in our method. Token compression enables richer information capture, optimized compression allocation improves efficiency, and extended context length significantly enhances performance.}
\label{tab:ablation}
\end{table*}
% \usepackage{multirow}
\begin{table*}[t]
\centering
\resizebox{0.9\textwidth}{!}{
\begin{tabular}{ccccccc}
\hline
{\color[HTML]{000000} } &  &  & \multicolumn{2}{c}{MLVU} & \multicolumn{2}{c}{LVBench} \\ \cline{4-7} 
\multirow{-2}{*}{{\color[HTML]{000000} Method}} & \multirow{-2}{*}{Max Frames} & \multirow{-2}{*}{Max Context Length} & NQA & AO & KIR & TG \\ \hline
LLaVA-Video-7B & 128 & 25K & 74.2 & 55.6 & 37.5 & 36.8 \\
\rowcolor[HTML]{EAFAF1} LLaVA-Video-7B+{\AdaReTaKe} & 1024 & 16K & 75.1 & 60.6 & 51.2 & 43.2 \\
Qwen2-VL-7B & 256 & 18K & 81.9 & 49.0 & 44.3 & 40.5 \\
\rowcolor[HTML]{EAFAF1} QWen2-VL-7B+{\AdaReTaKe} & 1024 & 16K & 82.7 & 60.2 & 52.9 & 42.7 \\ \hline
\end{tabular}
}

\caption{Ablation studies on MLVU and LVBench datasets, evaluating fine-grained perception capabilities across Needle QA (NQA), Action Order (AO), Action Count (AC), Key Information Retrieval (KIR), and Temporal Grounding (TG).}
\label{tab:paper_fine_tmp}
\end{table*}

% 引用 paper_comp 表格
% 要点1：{\AdaReTaKe}应用于各个baseline如LLaVA-Video和QWen2VL都有提升。说明了我们的方法的普适性。
% 要点2：对于VideoMME、MLVU和LVBench数据集，相对于baseline的平均提升分别为1.2%，2.8%，6.2%。考虑到LVBench有最长的平均长度，LVBench上提升幅度最大可能因为我们的方法有效压缩视觉token，使得MLLM看到了更长的视觉长度。
% 要点3：借助{\AdaReTaKe}，QWen2.5-VL 7B和72B分别在各自的模型规模取得SOTA结果。7B平均提升2.3%，72B平均提升1.5%。

% 5.2节可以分为两大段，第一段可以是”Comparision with SoTAs“，然后写整体达到的精度，把observation 2写进来，第二段可以写”Generalization for various MLLMs“，然后把observation 1和3写进来，分别是对多种模型和多种size的可适配性。 把5.2大段的题目改成Main Results。分成两段后一是贡献点清晰，二是读起来短段更好读
\noindent \textbf{Comparision with SoTAs.} We integrated {\AdaReTaKe} with various MLLMs and compared their results with existing long video understanding methods in \autoref{tab:comparison_mllm}. The average improvements on the VideoMME, MLVU, and LVBench datasets are 1.2\%, 2.8\%, and 6.2\%, respectively, with the most significant gains on LVBench. Given that LVBench has the longest average video duration (5x that of MLVU), we hypothesize that our method’s ability to effectively compress visual tokens enables MLLMs to process longer and more informative visual sequences, leading to greater improvements with longer video content.

\noindent \textbf{Generalization for various MLLMs.} When integrated with different MLLMs of different sizes, {\AdaReTaKe} brings consistent improvements, demonstrating its generality. With the help of {\AdaReTaKe}, both the 7B and 72B variants of QWen2.5-VL achieve state-of-the-art results within their respective model sizes. The 7B model sees an average improvement of 2.3\%, while the 72B model achieves a 1.5\% gain, demonstrating the scaling ability of our method into larger size.

% 把后面5.3节中的Comparison with other token compression methods段也提到5.2节中，作为第三段。这样布局更合适些
\noindent \textbf{Comparison with other token compression methods.} As shown in Table~\ref{tab:comparison_token_comp}, {\AdaReTaKe} demonstrates distinct advantages over existing MLLM token compression approaches.

Baseline methods FastV and FitPrune employ accumulated attention scores to evict tokens, while SparseVLM enhances this paradigm through partial token recycling. 

PyramidDrop, VL-Cache, and our method address compression ratio allocation. However, PyramidDrop's layer-wise monotonic budget allocation contradicts our layer importance observations in Section~\ref{sec:analysis}, leading to suboptimal performance. While VL-Cache improves through heuristic-based dynamic allocation, our method is theoretically grounded, achieving superior results.

% \noindent \textbf{Comparison with existing MLLMs.} We integrated {\AdaReTaKe} with various MLLMs and compared their results with existing long video understanding methods in \autoref{tab:comparison_mllm}. We gain several observations:
% 1) When integrated with different MLLMs of different sizes, {\AdaReTaKe} brings consistent improvements, demonstrating its generality.
% 2) The average improvements on the VideoMME, MLVU, and LVBench datasets are 1.2\%, 2.8\%, and 6.2\%, respectively, with the most significant gains on LVBench. Given that LVBench has the longest average video duration (5x that of MLVU), we hypothesize that our method’s ability to effectively compress visual tokens enables MLLMs to process longer and more informative visual sequences, leading to greater improvements with longer video content.
% 3) With the help of {\AdaReTaKe}, both the 7B and 72B variants of QWen2.5-VL achieve state-of-the-art results within their respective model sizes. The 7B model sees an average improvement of 2.3\%, while the 72B model achieves a 1.5\% gain, demonstrating the scaling ability of our method into larger size.

\subsection{Ablation Studies}\label{sec:ablation_studies}

\noindent \textbf{Ablation studies on temporal and layer-wise adaptive allocation.} 
% 引用 paper_ablation 表格
% 为了搞清楚模型方法提升的来源，我们做了消融实验
To identify the sources of performance improvements in our model, we conducted ablation studies, as summarized in Table~\ref{tab:ablation}.
% 其中序号零是baseline模型，#1是加入了token压缩的方法。序号2是单纯增加采样帧数，但是固定最大上下文长度。序号3是对不同的层应用不同的压缩率和pressure ratio。序号4是在不同的帧之间运用不同的压缩率compression ratio , #5 是增加上下文长度contact length。 
In the table, \#0 represents the baseline model. In \#1, we directly incorporates token compression into baseline model, and in \#2, we increase the number of sampled frames while keeping the maximum context length fixed. In \#3 and \#4 we applly varying compression ratios across different layers and different frames respectively. Finally, \#5 extends the context length.
%要点1：比较0、1和1、2，引入token compression会略微损失性能，但是同等context length下能够处理更多frames，捕捉更多有价值信息，能够显著提升性能，cover掉这部分损失（平均2.5% vs -1.1%）。
First, comparing rows 0,1 and 1,2 reveals that token compression introduces a slight performance drop (-0.8\% on average). However, it enables the model to process more frames within the same context length, capturing richer information and ultimately yielding a net performance gain (2.5\% on average versus -0.8\%).
% 要点2：比较2、3和3、4，我们在frame之间和layer之间分配compression ratio的方法，能够提升性能（平均1.1%和0.8%），验证了我们方法的有效性。
2) Comparing rows 2,3 and 3,4 shows that our strategy of distributing the compression ratio across frames and layers enhances performance (by 1.0\% and 0.8\% on average, respectively), confirming the effectiveness of our {\AdaReTaKe}.
% 要点3：比较4:5，通过进一步scaling context length到一般MLLM的上限\cite{shen_longvu_2024}，模型性能获得了显著的提升（平均3.6%）。
3) Comparing rows 4 and 5 demonstrates that scaling the context length to the typical upper limit of MLLMs \cite{shen_longvu_2024} further improves performance significantly, with an average gain of 3.4\%.

\noindent \textbf{Perception ability on temporal details.} 
% 引用 paper_fine_temp 表格
% 先说为了衡量细粒度感知能力，我们在MLVU和LVBench的以下分类做了消融实验。
% NQA， AO， AC， KIR， TG: Needle QA， Action Order，key information retrieval， temporal grounding
% 比较的基准模型是LLaVA-Video-7B和QWen2-VL-7B，分别在现存限制内采样尽可能多的frames（128和256）。
To assess the effectiveness of token compression algorithms in preserving critical temporal details, we conducted ablation studies on the MLVU and LVBench datasets, focusing on Needle QA, Action Order, Key Information Retrieval, and Temporal Grounding. We compared baseline models LLaVA-Video-7B and QWen2-VL-7B, maximizing frame sampling within their constraints (128 and 256 frames, respectively). Results are shown in Table~\ref{tab:paper_fine_tmp}. Our analysis reveals three key findings:
% 要点1：即使{\AdaReTaKe}进行了token压缩，最大采样帧数的增长依然带来了grounding能力提升（Needle QA， temporal grounding），并且并没有损失时间先后顺序感知能力（action order），说明{\AdaReTaKe}对模型性能的提升并没有牺牲细粒度时序能力，反而加强了它们。
1) Despite token compression via {\AdaReTaKe}, increasing the maximum sampled frames improved grounding abilities (Needle QA and Temporal Grounding) without compromising temporal order perception (Action Order). This indicates that {\AdaReTaKe} enhances model performance while strengthening fine-grained temporal capabilities. 
% 要点2：注意到MLVU的action order类别提升显著高于needle QA（平均8% vs 0.8%），我们分析这是由于我们的方法支持采样更多的frame，密集的frame采样对于动作理解帮助很大\cite{2023videochat}}, 
2) The improvement in MLVU's Action Order category was significantly higher than in Needle QA (8\% vs. 0.8\% on average). We attribute this to our method's ability to sample more frames through token compression, thus a denser frame sampling is enabled, which greatly enhances action understanding \cite{2023videochat}. 
% 要点3：注意到在LVBench数据集中，在基准相近的情况下，key information retrieval的提升远远大于temporal grounding（平均11.2% vs 4.3%）。我们猜测这是由于token compression增大了信息密度，使得key information retrieval这种综合性理解任务提升大于temporal grounding这种感知类任务。
3) In LVBench, under similar baselines, Key Information Retrieval demonstrated a significantly higher improvement compared to Temporal Grounding, with average gains of 11.2\% versus 4.3\%. We hypothesize that token compression enhances information density, which strengthens comprehensive understanding. We believe this can explain why Key Information Retrieval, a task requiring deeper comprehension, benefits more than perceptual tasks like Temporal Grounding in our results.

% \noindent \textbf{Comparison with other token compression methods.} 
% \TODO{@shiyu}
% 引用 paper_compare 表格
% 我们将{\AdaReTaKe}和其它适用于MLLM的token compression方法对比
% As shown in Table~\ref{tab:comparison_token_comp}, {\AdaReTaKe} demonstrates distinct advantages over existing MLLM token compression approaches.
% 要点1：FastV和FitPrune是基准方法，用accumulated attention score作为重要性分数进行token eviction以压缩序列。
% 要点2：Look-M和SparseVLM采用token recycle机制回收部分evicted token，有一定提升
% Baseline methods FastV and FitPrune employ accumulated attention scores to evict tokens, while SparseVLM enhances this paradigm through partial token recycling. 
% 要点3：PyramidDrop、VL-Cache和our {\AdaReTaKe}都在compression ratio分配机制。PyramidDrop在layer-wise单调分配预算，与我们Section~\ref{sec:analysis}的观察结果不符因而效果一般。VL-Cache采用启发规则的动态分配算法，而我们的方法有理论依据因而效果更好。
% PyramidDrop, VL-Cache, and our method address compression ratio allocation. However, PyramidDrop's layer-wise monotonic budget allocation contradicts our layer importance observations in Section~\ref{sec:analysis}, leading to suboptimal performance. While VL-Cache improves through heuristic-based dynamic allocation, our method is theoretically grounded, achieving superior results.%{\AdaReTaKe}'s theoretically grounded approach enables more principled compression decisions, ultimately achieving superior performance through its differentiable budget distribution mechanism.

% \subsection{Qualitative Analysis} 

\section{Conclusion}
% We present {\AdaReTaKe}, a training-free method for adaptive redundancy reduction in MLLMs. 
% By dynamically allocating compression ratios across frames and model layers, {\AdaReTaKe} performs better video token compression.
% Based on {\AdaReTaKe}, we are able to scale up more frames and perceive more valuable information with the original budget.
% Integrated into state-of-the-art MLLMs, it enables processing of up to 2,048 frames, outperforming existing methods on benchmarks like VideoMME, MLVU, LongVideoBench, and LVBench. 
% Our work highlights the importance of adaptive compression for advancing long video understanding, offering a practical solution to current computational limitations.

We introduce {\AdaReTaKe}, a training-free method for adaptive redundancy reduction in MLLMs. 
By dynamically allocating compression ratios across frames and model layers, {\AdaReTaKe} achieves more efficient video token compression. This allows us to scale up to more frames and extract valuable information within the same computational budget. 
Integrated into state-of-the-art MLLMs, {\AdaReTaKe} enables processing of up to 2,048 frames and outperforms existing methods on benchmarks such as VideoMME, MLVU, LongVideoBench, and LVBench by a large margin. 
% Our work underscores the importance of adaptive compression in advancing long video understanding and addresses current computational challenges effectively.

\section{Limitations}

While {\AdaReTaKe} can be integrated into most MLLMs, it may also inherit their inherent limitations, such as factual inaccuracies, biases, and hallucinations.

% Did you report the full text of instructions given to participants, including e.g., screenshots, disclaimers of any risks to participants or annotators, etc.?

\section{Acknowledgment}

This work was supported 
in part by the National Natural Science Foundation of China(Grant Nos. 62376069 and 62236003), 
in part by the Young Elite Scientists Sponsorship Program by CAST (Grant No. 2023QNRC001), 
in part by Guangdong Basic and Applied Basic Research Foundation (Grant No. 2024A1515012027), 
in part by Jiangsu Science and Technology Major Program (Grant No. BG2024041), 
and in part by the Shenzhen Science and Technology Program (Grant Nos. KQTD20240729102207002 and ZDSYS20230626091203008).

% \section*{Acknowledgments}

% This document has been adapted
% by Steven Bethard, Ryan Cotterell and Rui Yan
% from the instructions for earlier ACL and NAACL proceedings, including those for
% ACL 2019 by Douwe Kiela and Ivan Vuli\'{c},
% NAACL 2019 by Stephanie Lukin and Alla Roskovskaya,
% ACL 2018 by Shay Cohen, Kevin Gimpel, and Wei Lu,
% NAACL 2018 by Margaret Mitchell and Stephanie Lukin,
% Bib\TeX{} suggestions for (NA)ACL 2017/2018 from Jason Eisner,
% ACL 2017 by Dan Gildea and Min-Yen Kan,
% NAACL 2017 by Margaret Mitchell,
% ACL 2012 by Maggie Li and Michael White,
% ACL 2010 by Jing-Shin Chang and Philipp Koehn,
% ACL 2008 by Johanna D. Moore, Simone Teufel, James Allan, and Sadaoki Furui,
% ACL 2005 by Hwee Tou Ng and Kemal Oflazer,
% ACL 2002 by Eugene Charniak and Dekang Lin,
% and earlier ACL and EACL formats written by several people, including
% John Chen, Henry S. Thompson and Donald Walker.
% Additional elements were taken from the formatting instructions of the \emph{International Joint Conference on Artificial Intelligence} and the \emph{Conference on Computer Vision and Pattern Recognition}.

\newpage
\bibliography{main}

\begin{thebibliography}{49}
\providecommand{\natexlab}[1]{#1}

\bibitem[{Bolya et~al.(2022)Bolya, Fu, Dai, Zhang, Feichtenhofer, and Hoffman}]{bolya_tome_2022}
Daniel Bolya, Cheng-Yang Fu, Xiaoliang Dai, Peizhao Zhang, Christoph Feichtenhofer, and Judy Hoffman. 2022.
\newblock Token merging: Your vit but faster.
\newblock In \emph{International Conference on Learning Representations}. OpenReview.net.

\bibitem[{Chen et~al.(2024)Chen, Zhao, Liu, Bai, Lin, Zhou, and Chang}]{chen_fastv_2024}
Liang Chen, Haozhe Zhao, Tianyu Liu, Shuai Bai, Junyang Lin, Chang Zhou, and Baobao Chang. 2024.
\newblock An {Image} is {Worth} 1/2 {Tokens} {After} {Layer} 2: {Plug}-and-{Play} {Inference} {Acceleration} for {Large} {Vision}-{Language} {Models}.
\newblock In \emph{{European} {Conference} on {Computer} {Vision}}, Lecture {Notes} in {Computer} {Science}, pages 19--35. Springer.

\bibitem[{Cheng et~al.(2024)Cheng, Li, Liu, Guo, Jiang, Liu, Chen, and Zhao}]{cheng_he_mllm_2024}
Dingxin Cheng, Mingda Li, Jingyu Liu, Yongxin Guo, Bin Jiang, Qingbin Liu, Xi~Chen, and Bo~Zhao. 2024.
\newblock Enhancing {Long} {Video} {Understanding} via {Hierarchical} {Event}-{Based} {Memory}.
\newblock ArXiv:2409.06299.

\bibitem[{Fei et~al.(2024)Fei, Li, Deng, Wang, Liu, and Wang}]{fei_video-ccam_2024}
Jiajun Fei, Dian Li, Zhidong Deng, Zekun Wang, Gang Liu, and Hui Wang. 2024.
\newblock Video-{CCAM}: {Enhancing} {Video}-{Language} {Understanding} with {Causal} {Cross}-{Attention} {Masks} for {Short} and {Long} {Videos}.
\newblock ArXiv:2408.14023.

\bibitem[{Feng et~al.(2024)Feng, Lv, Cao, Xie, and Zhou}]{feng_ada-kv_2024}
Yuan Feng, Junlin Lv, Yukun Cao, Xike Xie, and S.~Kevin Zhou. 2024.
\newblock Ada-{KV}: {Optimizing} {KV} {Cache} {Eviction} by {Adaptive} {Budget} {Allocation} for {Efficient} {LLM} {Inference}.
\newblock \emph{arXiv preprint}.
\newblock ArXiv:2407.11550.

\bibitem[{Fu et~al.(2024)Fu, Dai, Luo, Li, Ren, Zhang, Wang, Zhou, Shen, Zhang, Chen, Li, Lin, Zhao, Li, Xu, Zheng, Chen, Ji, and Sun}]{fu_videomme_2024}
Chaoyou Fu, Yuhan Dai, Yondong Luo, Lei Li, Shuhuai Ren, Renrui Zhang, Zihan Wang, Chenyu Zhou, Yunhang Shen, Mengdan Zhang, Peixian Chen, Yanwei Li, Shaohui Lin, Sirui Zhao, Ke~Li, Tong Xu, Xiawu Zheng, Enhong Chen, Rongrong Ji, and Xing Sun. 2024.
\newblock Video-mme: The first-ever comprehensive evaluation benchmark of multi-modal llms in video analysis.
\newblock \emph{CoRR}, abs/2405.21075.

\bibitem[{Gan et~al.(2023)Gan, Wang, Sun, Wu, Guo, and Nie}]{gan2023temporal}
Tian Gan, Xiao Wang, Yan Sun, Jianlong Wu, Qingpei Guo, and Liqiang Nie. 2023.
\newblock Temporal sentence grounding in streaming videos.
\newblock In \emph{Proceedings of the 31st ACM International Conference on Multimedia}, pages 4637--4646.

\bibitem[{Han et~al.(2024)Han, Guo, Pan, Liu, Guan, and Yang}]{han_dynfocus_2024}
Yudong Han, Qingpei Guo, Liyuan Pan, Liu Liu, Yu~Guan, and Ming Yang. 2024.
\newblock {DynFocus}: {Dynamic} {Cooperative} {Network} {Empowers} {LLMs} with {Video} {Understanding}.
\newblock ArXiv:2411.12355.

\bibitem[{He et~al.(2024{\natexlab{a}})He, Li, Jang, Jia, Cao, Shah, Shrivastava, and Lim}]{he_ma-lmm_2024}
Bo~He, Hengduo Li, Young~Kyun Jang, Menglin Jia, Xuefei Cao, Ashish Shah, Abhinav Shrivastava, and Ser-Nam Lim. 2024{\natexlab{a}}.
\newblock {MA}-{LMM}: {Memory}-{Augmented} {Large} {Multimodal} {Model} for {Long}-{Term} {Video} {Understanding}.
\newblock In \emph{{Conference} on {Computer} {Vision} and {Pattern} {Recognition}}, pages 13504--13514. IEEE.

\bibitem[{He et~al.(2024{\natexlab{b}})He, Chen, Liu, Shao, Zhou, Zhang, and Zhuang}]{he_zipvl_2024}
Yefei He, Feng Chen, Jing Liu, Wenqi Shao, Hong Zhou, Kaipeng Zhang, and Bohan Zhuang. 2024{\natexlab{b}}.
\newblock {ZipVL}: {Efficient} {Large} {Vision}-{Language} {Models} with {Dynamic} {Token} {Sparsification} and {KV} {Cache} {Compression}.
\newblock ArXiv:2410.08584.

\bibitem[{Hooper et~al.(2024)Hooper, Kim, Mohammadzadeh, Mahoney, Shao, Keutzer, and Gholami}]{coleman_kvquant_2024}
Coleman Hooper, Sehoon Kim, Hiva Mohammadzadeh, Michael~W. Mahoney, Yakun~Sophia Shao, Kurt Keutzer, and Amir Gholami. 2024.
\newblock Kvquant: Towards 10 million context length {LLM} inference with {KV} cache quantization.
\newblock In \emph{Advances in Neural Information Processing Systems}.

\bibitem[{Li et~al.(2024{\natexlab{a}})Li, Zhang, Guo, Zhang, Li, Zhang, Zhang, Li, Liu, and Li}]{li_llava-onevision_2024}
Bo~Li, Yuanhan Zhang, Dong Guo, Renrui Zhang, Feng Li, Hao Zhang, Kaichen Zhang, Yanwei Li, Ziwei Liu, and Chunyuan Li. 2024{\natexlab{a}}.
\newblock {LLaVA}-{OneVision}: {Easy} {Visual} {Task} {Transfer}.
\newblock ArXiv:2408.03326.

\bibitem[{Li et~al.(2023{\natexlab{a}})Li, Li, Savarese, and Hoi}]{li_blip-2_2023}
Junnan Li, Dongxu Li, Silvio Savarese, and Steven C.~H. Hoi. 2023{\natexlab{a}}.
\newblock {BLIP}-2: {Bootstrapping} {Language}-{Image} {Pre}-training with {Frozen} {Image} {Encoders} and {Large} {Language} {Models}.
\newblock In \emph{International {Conference} on {Machine} {Learning}}, pages 19730--19742. PMLR.

\bibitem[{Li et~al.(2024{\natexlab{b}})Li, He, Wang, Li, Wang, Luo, Wang, Wang, and Qiao}]{2023videochat}
KunChang Li, Yinan He, Yi~Wang, Yizhuo Li, Wenhai Wang, Ping Luo, Yali Wang, Limin Wang, and Yu~Qiao. 2024{\natexlab{b}}.
\newblock Videochat: Chat-centric video understanding.
\newblock ArXiv:2305.06355.

\bibitem[{Li et~al.(2023{\natexlab{b}})Li, Xue, Baranwal, Li, and You}]{li_sp_2023}
Shenggui Li, Fuzhao Xue, Chaitanya Baranwal, Yongbin Li, and Yang You. 2023{\natexlab{b}}.
\newblock Sequence parallelism: Long sequence training from system perspective.
\newblock In \emph{Proceedings of the Annual Meeting of the Association for Computational Linguistics}, pages 2391--2404. Association for Computational Linguistics.

\bibitem[{Lin et~al.(2024)Lin, Ye, Zhu, Cui, Ning, Jin, and Yuan}]{bin_videollava_2024}
Bin Lin, Yang Ye, Bin Zhu, Jiaxi Cui, Munan Ning, Peng Jin, and Li~Yuan. 2024.
\newblock Video-llava: Learning united visual representation by alignment before projection.
\newblock In \emph{Proceedings of the Conference on Empirical Methods in Natural Language Processing}, pages 5971--5984. Association for Computational Linguistics.

\bibitem[{Liu et~al.(2018)Liu, Wang, Nie, He, Chen, and Chua}]{liu2018attentive}
Meng Liu, Xiang Wang, Liqiang Nie, Xiangnan He, Baoquan Chen, and Tat-Seng Chua. 2018.
\newblock Attentive moment retrieval in videos.
\newblock In \emph{The 41st international ACM SIGIR conference on research \& development in information retrieval}, pages 15--24.

\bibitem[{Liu et~al.(2024)Liu, Shi, Hong, Hu, Yin, and Zhang}]{liu_mustdrop_2024}
Ting Liu, Liangtao Shi, Richang Hong, Yue Hu, Quanjun Yin, and Linfeng Zhang. 2024.
\newblock Multi-{Stage} {Vision} {Token} {Dropping}: {Towards} {Efficient} {Multimodal} {Large} {Language} {Model}.
\newblock ArXiv:2411.10803.

\bibitem[{Luo et~al.(2024)Luo, Zheng, Yang, Li, Lin, Huang, Ji, Chao, Luo, and Ji}]{luo_video_rag_2024}
Yongdong Luo, Xiawu Zheng, Xiao Yang, Guilin Li, Haojia Lin, Jinfa Huang, Jiayi Ji, Fei Chao, Jiebo Luo, and Rongrong Ji. 2024.
\newblock Video-{RAG}: {Visually}-aligned {Retrieval}-{Augmented} {Long} {Video} {Comprehension}.
\newblock ArXiv:2411.13093.

\bibitem[{Man et~al.(2024)Man, Huang, Zhang, Li, Niu, and Yin}]{man_adacm2_2024}
Yuanbin Man, Ying Huang, Chengming Zhang, Bingzhe Li, Wei Niu, and Miao Yin. 2024.
\newblock {AdaCM}\${\textasciicircum}2\$: {On} {Understanding} {Extremely} {Long}-{Term} {Video} with {Adaptive} {Cross}-{Modality} {Memory} {Reduction}.
\newblock ArXiv:2411.12593.

\bibitem[{Nemhauser et~al.(1978)Nemhauser, Wolsey, and Fisher}]{nemhauser1978_submodular_analysis}
George~L Nemhauser, Laurence~A Wolsey, and Marshall~L Fisher. 1978.
\newblock An analysis of approximations for maximizing submodular set functions—i.
\newblock \emph{Mathematical programming}, 14:265--294.

\bibitem[{Shang et~al.(2024)Shang, Xu, Kang, Cai, Li, Wen, Dong, Keutzer, Lee, and Yan}]{shang_intp_2024}
Yuzhang Shang, Bingxin Xu, Weitai Kang, Mu~Cai, Yuheng Li, Zehao Wen, Zhen Dong, Kurt Keutzer, Yong~Jae Lee, and Yan Yan. 2024.
\newblock Interpolating {Video}-{LLMs}: {Toward} {Longer}-sequence {LMMs} in a {Training}-free {Manner}.
\newblock ArXiv:2409.12963.

\bibitem[{Shen et~al.(2024)Shen, Xiong, Zhao, Wu, Chen, Zhu, Liu, Xiao, Varadarajan, Bordes, Liu, Xu, Kim, Soran, Krishnamoorthi, Elhoseiny, and Chandra}]{shen_longvu_2024}
Xiaoqian Shen, Yunyang Xiong, Changsheng Zhao, Lemeng Wu, Jun Chen, Chenchen Zhu, Zechun Liu, Fanyi Xiao, Balakrishnan Varadarajan, Florian Bordes, Zhuang Liu, Hu~Xu, Hyunwoo~J. Kim, Bilge Soran, Raghuraman Krishnamoorthi, Mohamed Elhoseiny, and Vikas Chandra. 2024.
\newblock {LongVU}: {Spatiotemporal} {Adaptive} {Compression} for {Long} {Video}-{Language} {Understanding}.
\newblock ArXiv:2410.17434.

\bibitem[{Shu et~al.(2024)Shu, Zhang, Liu, Qin, Zhou, Huang, and Zhao}]{shu_video-xl_2024}
Yan Shu, Peitian Zhang, Zheng Liu, Minghao Qin, Junjie Zhou, Tiejun Huang, and Bo~Zhao. 2024.
\newblock Video-{XL}: {Extra}-{Long} {Vision} {Language} {Model} for {Hour}-{Scale} {Video} {Understanding}.
\newblock ArXiv:2409.14485.

\bibitem[{Tu et~al.(2024)Tu, Vashchilenko, Lu, and Xu}]{tu_vl-cache_2024}
Dezhan Tu, Danylo Vashchilenko, Yuzhe Lu, and Panpan Xu. 2024.
\newblock {VL}-{Cache}: {Sparsity} and {Modality}-{Aware} {KV} {Cache} {Compression} for {Vision}-{Language} {Model} {Inference} {Acceleration}.
\newblock ArXiv:2410.23317.

\bibitem[{Wan et~al.(2024)Wan, Wu, Liu, Huang, Zhu, Jin, Wang, and Yuan}]{wan_look-m_2024}
Zhongwei Wan, Ziang Wu, Che Liu, Jinfa Huang, Zhihong Zhu, Peng Jin, Longyue Wang, and Li~Yuan. 2024.
\newblock {LOOK}-{M}: {Look}-{Once} {Optimization} in {KV} {Cache} for {Efficient} {Multimodal} {Long}-{Context} {Inference}.
\newblock ArXiv: 2406.18139.

\bibitem[{Wang et~al.(2024{\natexlab{a}})Wang, Bai, Tan, Wang, Fan, Bai, Chen, Liu, Wang, Ge, Fan, Dang, Du, Ren, Men, Liu, Zhou, Zhou, and Lin}]{wang_qwen2vl_2024}
Peng Wang, Shuai Bai, Sinan Tan, Shijie Wang, Zhihao Fan, Jinze Bai, Keqin Chen, Xuejing Liu, Jialin Wang, Wenbin Ge, Yang Fan, Kai Dang, Mengfei Du, Xuancheng Ren, Rui Men, Dayiheng Liu, Chang Zhou, Jingren Zhou, and Junyang Lin. 2024{\natexlab{a}}.
\newblock Qwen2-vl: Enhancing vision-language model's perception of the world at any resolution.
\newblock ArXiv:2409.12191.

\bibitem[{Wang et~al.(2024{\natexlab{b}})Wang, He, Hong, Cheng, Zhang, Qi, Huang, Xu, Dong, Ding, and Tang}]{wang_lvbench_2024}
Weihan Wang, Zehai He, Wenyi Hong, Yean Cheng, Xiaohan Zhang, Ji~Qi, Shiyu Huang, Bin Xu, Yuxiao Dong, Ming Ding, and Jie Tang. 2024{\natexlab{b}}.
\newblock Lvbench: An extreme long video understanding benchmark.

\bibitem[{Wang et~al.(2025{\natexlab{a}})Wang, Hua, Lin, Zhang, Zhang, Wu, Zhang, and Nie}]{wang2025haic}
Xiao Wang, Jingyun Hua, Weihong Lin, Yuanxing Zhang, Fuzheng Zhang, Jianlong Wu, Di~Zhang, and Liqiang Nie. 2025{\natexlab{a}}.
\newblock Haic: Improving human action understanding and generation with better captions for multi-modal large language models.
\newblock \emph{arXiv preprint ArXiv:2502.20811}.

\bibitem[{Wang et~al.(2023)Wang, Li, Gan, Zhang, Lv, and Nie}]{wang2023rtq}
Xiao Wang, Yaoyu Li, Tian Gan, Zheng Zhang, Jingjing Lv, and Liqiang Nie. 2023.
\newblock Rtq: Rethinking video-language understanding based on image-text model.
\newblock In \emph{Proceedings of the 31st ACM International Conference on Multimedia}, pages 557--566.

\bibitem[{Wang et~al.(2024{\natexlab{c}})Wang, Si, Wu, Zhu, Cao, and Nie}]{xiao_retake_2024}
Xiao Wang, Qingyi Si, Jianlong Wu, Shiyu Zhu, Li~Cao, and Liqiang Nie. 2024{\natexlab{c}}.
\newblock Retake: Reducing temporal and knowledge redundancy for long video understanding.
\newblock ArXiv:2412.20504.

\bibitem[{Wang et~al.(2025{\natexlab{b}})Wang, Wu, Lin, Zhang, Zhang, and Nie}]{wang_vid_dataflywheel_2024}
Xiao Wang, Jianlong Wu, Zijia Lin, Fuzheng Zhang, Di~Zhang, and Liqiang Nie. 2025{\natexlab{b}}.
\newblock Video dataflywheel: Resolving the impossible data trinity in video-language understanding.
\newblock \emph{IEEE Transactions on Pattern Analysis and Machine Intelligence}.

\bibitem[{Wang et~al.(2024{\natexlab{d}})Wang, Zhang, Zohar, and Yeung-Levy}]{wang_videoagent_2024}
Xiaohan Wang, Yuhui Zhang, Orr Zohar, and Serena Yeung-Levy. 2024{\natexlab{d}}.
\newblock {VideoAgent}: {Long}-{Form} {Video} {Understanding} with {Large} {Language} {Model} as {Agent}.
\newblock In \emph{{European} {Conference} on {Computer} {Vision}}, volume 15138, pages 58--76. Springer.

\bibitem[{Wei and Chen(2024)}]{wei_visyarn_2024}
Hongchen Wei and Zhenzhong Chen. 2024.
\newblock Visual {Context} {Window} {Extension}: {A} {New} {Perspective} for {Long} {Video} {Understanding}.
\newblock ArXiv:2409.20018.

\bibitem[{Wu et~al.(2024)Wu, Li, Chen, and Li}]{wu_longvideobench_2024}
Haoning Wu, Dongxu Li, Bei Chen, and Junnan Li. 2024.
\newblock Longvideobench: {A} benchmark for long-context interleaved video-language understanding.
\newblock \emph{CoRR}, abs/2407.15754.

\bibitem[{Xiao et~al.(2024)Xiao, Tian, Chen, Han, and Lewis}]{xiao_streamingllm_2024}
Guangxuan Xiao, Yuandong Tian, Beidi Chen, Song Han, and Mike Lewis. 2024.
\newblock Efficient streaming language models with attention sinks.
\newblock In \emph{International Conference on Learning Representations}. OpenReview.net.

\bibitem[{Xing et~al.(2024)Xing, Huang, Dong, Lu, Zhang, Zang, Cao, He, Wang, Wu, and Lin}]{xing_pyramiddrop_2024}
Long Xing, Qidong Huang, Xiaoyi Dong, Jiajie Lu, Pan Zhang, Yuhang Zang, Yuhang Cao, Conghui He, Jiaqi Wang, Feng Wu, and Dahua Lin. 2024.
\newblock {PyramidDrop}: {Accelerating} {Your} {Large} {Vision}-{Language} {Models} via {Pyramid} {Visual} {Redundancy} {Reduction}.
\newblock ArXiv:2410.17247.

\bibitem[{Xue et~al.(2024)Xue, Chen, Li, Hu, Zhu, Li, Fang, Tang, Yang, Liu, He, Yin, Molchanov, Kautz, Fan, Zhu, Lu, and Han}]{xue_longvila_2024}
Fuzhao Xue, Yukang Chen, Dacheng Li, Qinghao Hu, Ligeng Zhu, Xiuyu Li, Yunhao Fang, Haotian Tang, Shang Yang, Zhijian Liu, Ethan He, Hongxu Yin, Pavlo Molchanov, Jan Kautz, Linxi Fan, Yuke Zhu, Yao Lu, and Song Han. 2024.
\newblock {LongVILA}: {Scaling} {Long}-{Context} {Visual} {Language} {Models} for {Long} {Videos}.
\newblock ArXiv:2408.10188.

\bibitem[{Ye et~al.(2024)Ye, Wu, Lin, and Zhou}]{ye_fitprune_2024}
Weihao Ye, Qiong Wu, Wenhao Lin, and Yiyi Zhou. 2024.
\newblock Fit and {Prune}: {Fast} and {Training}-free {Visual} {Token} {Pruning} for {Multi}-modal {Large} {Language} {Models}.
\newblock ArXiv:2409.10197.

\bibitem[{Zeng et~al.(2024{\natexlab{a}})Zeng, Li, Wang, Li, Jiang, Yan, Li, Shi, Yue, Wang, Wang, Qiao, and Wang}]{zeng_timesuite_2024}
Xiangyu Zeng, Kunchang Li, Chenting Wang, Xinhao Li, Tianxiang Jiang, Ziang Yan, Songze Li, Yansong Shi, Zhengrong Yue, Yi~Wang, Yali Wang, Yu~Qiao, and Limin Wang. 2024{\natexlab{a}}.
\newblock {TimeSuite}: {Improving} {MLLMs} for {Long} {Video} {Understanding} via {Grounded} {Tuning}.
\newblock ArXiv:2410.19702.

\bibitem[{Zeng et~al.(2024{\natexlab{b}})Zeng, Guo, Liu, Yin, Shu, Huang, Wang, Zhou, Li, Liu et~al.}]{zeng_chunked_prefill_2024}
Zhiyuan Zeng, Qipeng Guo, Xiaoran Liu, Zhangyue Yin, Wentao Shu, Mianqiu Huang, Bo~Wang, Yunhua Zhou, Linlin Li, Qun Liu, et~al. 2024{\natexlab{b}}.
\newblock Memorize step by step: Efficient long-context prefilling with incremental memory and decremental chunk.
\newblock In \emph{Proceedings of the Conference on Empirical Methods in Natural Language Processing}, pages 21021--21034. ACL.

\bibitem[{Zhang et~al.(2024{\natexlab{a}})Zhang, Lu, Islam, Wang, Yu, Bansal, and Bertasius}]{zhang_llovi_2024}
Ce~Zhang, Taixi Lu, Md~Mohaiminul Islam, Ziyang Wang, Shoubin Yu, Mohit Bansal, and Gedas Bertasius. 2024{\natexlab{a}}.
\newblock A {Simple} {LLM} {Framework} for {Long}-{Range} {Video} {Question}-{Answering}.
\newblock In \emph{Proceedings of the {Conference} on {Empirical} {Methods} in {Natural} {Language} {Processing}}, pages 21715--21737. Association for Computational Linguistics.

\bibitem[{Zhang et~al.(2024{\natexlab{b}})Zhang, Li, Zhang, Pu, Cahyono, Hu, Liu, Zhang, Yang, Li, and Liu}]{zhang_2024_lmmseval}
Kaichen Zhang, Bo~Li, Peiyuan Zhang, Fanyi Pu, Joshua~Adrian Cahyono, Kairui Hu, Shuai Liu, Yuanhan Zhang, Jingkang Yang, Chunyuan Li, and Ziwei Liu. 2024{\natexlab{b}}.
\newblock \href {https://arxiv.org/abs/2407.12772} {Lmms-eval: Reality check on the evaluation of large multimodal models}.
\newblock \emph{Preprint}, arXiv:2407.12772.

\bibitem[{Zhang et~al.(2024{\natexlab{c}})Zhang, Zhang, Li, Zeng, Yang, Zhang, Wang, Tan, Li, and Liu}]{zhang_longva_2024}
Peiyuan Zhang, Kaichen Zhang, Bo~Li, Guangtao Zeng, Jingkang Yang, Yuanhan Zhang, Ziyue Wang, Haoran Tan, Chunyuan Li, and Ziwei Liu. 2024{\natexlab{c}}.
\newblock Long {Context} {Transfer} from {Language} to {Vision}.
\newblock ArXiv:2406.16852.

\bibitem[{Zhang et~al.(2024{\natexlab{d}})Zhang, Fan, Ma, Zheng, Huang, Cheng, Gudovskiy, Okuno, Nakata, Keutzer, and Zhang}]{zhang_sparsevlm_2024}
Yuan Zhang, Chun-Kai Fan, Junpeng Ma, Wenzhao Zheng, Tao Huang, Kuan Cheng, Denis Gudovskiy, Tomoyuki Okuno, Yohei Nakata, Kurt Keutzer, and Shanghang Zhang. 2024{\natexlab{d}}.
\newblock {SparseVLM}: {Visual} {Token} {Sparsification} for {Efficient} {Vision}-{Language} {Model} {Inference}.
\newblock ArXiv:2410.04417.

\bibitem[{Zhang et~al.(2024{\natexlab{e}})Zhang, Wu, Li, Li, Ma, Liu, and Li}]{zhang_llava-video_2024}
Yuanhan Zhang, Jinming Wu, Wei Li, Bo~Li, Zejun Ma, Ziwei Liu, and Chunyuan Li. 2024{\natexlab{e}}.
\newblock Video {Instruction} {Tuning} {With} {Synthetic} {Data}.
\newblock ArXiv:2410.02713.

\bibitem[{Zhang et~al.(2023)Zhang, Sheng, Zhou, Chen, Zheng, Cai, Song, Tian, Ré, Barrett, Wang, and Chen}]{zhang_h2o_2023}
Zhenyu Zhang, Ying Sheng, Tianyi Zhou, Tianlong Chen, Lianmin Zheng, Ruisi Cai, Zhao Song, Yuandong Tian, Christopher Ré, Clark~W. Barrett, Zhangyang Wang, and Beidi Chen. 2023.
\newblock {H2O}: {Heavy}-{Hitter} {Oracle} for {Efficient} {Generative} {Inference} of {Large} {Language} {Models}.
\newblock In \emph{Advances in {Neural} {Information} {Processing} {Systems}3}.

\bibitem[{Zhou et~al.(2024)Zhou, Shu, Zhao, Wu, Xiao, Yang, Xiong, Zhang, Huang, and Liu}]{zhou_mlvu_2024}
Junjie Zhou, Yan Shu, Bo~Zhao, Boya Wu, Shitao Xiao, Xi~Yang, Yongping Xiong, Bo~Zhang, Tiejun Huang, and Zheng Liu. 2024.
\newblock {MLVU:} {A} comprehensive benchmark for multi-task long video understanding.
\newblock \emph{CoRR}, abs/2406.04264.

\bibitem[{Zhu et~al.(2024)Zhu, Xie, Liang, Zheng, and Guo}]{zhu_focusllava_2024}
Yuke Zhu, Chi Xie, Shuang Liang, Bo~Zheng, and Sheng Guo. 2024.
\newblock {FocusLLaVA}: {A} {Coarse}-to-{Fine} {Approach} for {Efficient} and {Effective} {Visual} {Token} {Compression}.
\newblock ArXiv:2411.14228.

\end{thebibliography}

\appendix

\section{Proof of Token Compression Loss} \label{sec:proof_token_comp}

\subsection{Preliminaries}

LLMs are characterized by an autoregressive generation mode, where each step involves using the last token to predict the next token. For a formal representation of our approach, we denote $\mathbf{K}^{(l)}$,$\mathbf{V}^{(l)}\in\mathbb{R}^{(n-1)\times d},i=1,2,...,L$ as the video KV cache in and $\mathbf{x}^{(l)}\in\mathbb{R}^d$ as the last token used as input at the current time step, where $n-1$ is the prefilled sequence length, $L$ is the number of layers and $d$ is the hidden state dimension. Note that we ignore the number of attention heads for simplification.

During attention forward in the $l$-th layer, the input token is first mapped into its query, key, and value states:
\begin{empheq}[left=\empheqlbrace]{align}
    \mathbf{q}^{(l)} &= \mathbf{x}^{(l)}\mathbf{W}_Q^{(l)}, \\
    \mathbf{k}^{(l)} &= \mathbf{x}^{(l)}\mathbf{W}_K^{(l)}, \\
    \mathbf{v}^{(l)} &= \mathbf{x}^{(l)}\mathbf{W}_V^{(l)},
\end{empheq}
where $\mathbf{W}_Q^{(l)},\mathbf{W}_K^{(l)},\mathbf{W}_V^{(l)}\in\mathbb{R}^{d\times d}$ are transformation matrices.
Then, the previous KV Cache is updated:
\begin{empheq}[left=\empheqlbrace]{align}
    \mathbf{K}^{(l)} &\gets \left[\mathbf{K}^{(l)} || \mathbf{k}^{(l)}\right], \\
    \mathbf{V}^{(l)} &\gets \left[\mathbf{V}^{(l)} || \mathbf{v}^{(l)}\right],
\end{empheq}
where $||$ denotes vector concatenation.
Finally, the output of the $l$-th layer $\mathbf{y}^{(l)}\in\mathbb{R}^d$ is computed:
\begin{empheq}[left=\empheqlbrace]{align}
    \mathbf{y}^{(l)} &= \mathbf{A}^{(l)}\mathbf{V}^{(l)}\mathbf{O}_V^{(l)}, \\
    \mathbf{A}^{(l)} &= \textmd{Softmax}\left(\mathbf{q}^{(l)}(\mathbf{K}^{(l)})^T\right), \label{eq:ori_attn_matrix}
\end{empheq}
where $\mathbf{O}_V^{(l)}$ is the transformation matrics, $\mathbf{A}^{(l)}$ is the attention weights.

\subsection{Attention Output after Compression}

We define \textbf{token compression choice} $\mathbf{I}^{(l)}\in\{0,1\}^{n}$ to represent which token to preserve during compression, where the value of its $i$-th element satisfies:
\begin{equation}\label{eq:definition_I}
    I_i^{(l)} = 
    \begin{cases} 
    1, & \text{if } K_i^{(l)} \text{ and } V_i^{(l)} \text{ are retained}, \\
    0, & \text{if } K_i^{(l)} \text{ and } V_i^{(l)} \text{ are dropped},
    \end{cases}
\end{equation}
where $K_i^{(l)}$ and $V_i^{(l)}$ is the $i$-th element in $\textbf{K}^{(l)}$ and $\textbf{V}^{(l)}$, respectively.
Based on this, the output of the $l$-th layer after compression $\hat{\mathbf{y}}^{(l)}\in\mathbb{R}^d$ can be represented in a concise format according to the lemma below. Note that this lemma borrows insights from Theorem 1 in AdaKV \cite{feng_ada-kv_2024}, we extend it to the multi-layer scenario.

\begin{lemma}\label{lemma:output_after_comp}
Given token compression choice $\mathbf{I}^{(l)}$, the output of the $l$-th layer after compression $\hat{y}^{(l)}$ can be rewritten as:
\begin{equation}
    \hat{\mathbf{y}}^{(l)} =  \frac{\hat{\textbf{A}}^{(l)} \odot \mathbf{I}^{(l)}}{||\hat{\textbf{A}}^{(l)} \odot \mathbf{I}^{(l)}||_1} \textbf{V}^{(l)} \textbf{W}_O^{(l)},
\end{equation}
where $\hat{\textbf{A}}^{(l)}=\hat{\mathbf{q}}^{(l)}(\mathbf{K}^{(l)})^T$, $\hat{\mathbf{q}}^{(l)}$ is the input of the $l$-th attention layer after KV cache in the $(l-1)$-th layer is compressed.
\end{lemma}

\begin{proof}
Considering that Softmax function is:
\begin{equation}
    \text{Softmax}(x)_j = \frac{\exp(x_j)}{\sum_j \exp(x+j)},
\end{equation}
we can use $-\infty$ to represent a token that is dropped during softmax calculation. Thus, the attention weight in the $l$-th layer after token compression is:
\begin{equation}\label{eq:attn_weights_after_comp_ori}
    \tilde{\textbf{A}}^{(l)} = \text{Softmax}\left(-\infty \odot (\mathbf{1} - \mathbf{I}^{(l)}) + \hat{\textbf{s}}^{(l)}\right),
\end{equation}
where $\hat{\textbf{s}}^{(l)}$ is the attention logits before compression:
\begin{equation}
    \hat{\textbf{s}}^{(l)} = \hat{\mathbf{q}}^{(l)}(\mathbf{K}^{(l)})^T.
\end{equation}
The $i$-th element of $\tilde{\textbf{A}}^{(l)}$ is:
\begin{align}
    \tilde{A}_i^{(l)} &= \frac{\exp(\hat{s}_i^{(l)} - \infty \odot (1 - I_i^{(l)}))}{\sum_j \exp(\hat{s}_j^{(l)} - \infty \odot (1 - I_j^{(l)}))}, \\
                &= \frac{I_i^{(l)} \exp(\hat{s}_i^{(l)})}{\sum_j I_j^{(l)} \exp(\hat{s}_j^{(l)})}, \\
                &= \frac{I_i^{(l)} \exp(\hat{s}_i^{(l)})}{\sum_j \exp(\hat{s}_j^{(l)})} \frac{\sum_j \exp(\hat{s}_j^{(l)})}{\sum_j I_j^{(l)} \exp(\hat{s}_j^{(l)})},
\end{align}
where $\hat{s}_i^{(l)}$ is the $i$-th element in $\hat{\textbf{s}}^{(l)}$.

Denote $\hat{\textbf{A}}^{(l)}=\hat{\mathbf{q}}^{(l)}(\mathbf{K}^{(l)})^T$, we can get further simplify $\tilde{A}_i^{(l)}$ above:
\begin{align}
    \tilde{A}_i^{(l)} &= I_i^{(l)} \hat{A}_i^{(l)} \frac{\sum_j \exp(\hat{s}_j^{(l)})}{\sum_j I_j^{(l)} \exp(\hat{s}_j^{(l)})}, \\
                &= \frac{I_i^{(l)} \hat{A}_i^{(l)}}{||\hat{\textbf{A}}^{(l)} \odot \textbf{I}^{(l)}||_1},
\end{align}
where $\hat{A}_i^{(l)}$ is the $i$-th element in $\hat{\textbf{A}}^{(l)}$.

Then we can obtain the simplified form of the attention weight after token compression in Eqn.~(\ref{eq:attn_weights_after_comp_ori}):
\begin{equation}
    \tilde{\textbf{A}}^{(l)} = \frac{\hat{\textbf{A}}^{(l)} \odot \mathbf{I}^{(l)}}{||\hat{\textbf{A}}^{(l)} \odot \mathbf{I}^{(l)}||_1}.
\end{equation}

Thus:
\begin{equation}
    \hat{\textbf{y}}^{(l)} = \frac{\hat{\textbf{A}}^{(l)} \odot \mathbf{I}^{(l)}}{||\hat{\textbf{A}}^{(l)} \odot \mathbf{I}^{(l)}||_1} \textbf{V}^{(l)} \textbf{W}_O^{(l)}.
\end{equation}
\end{proof}

\subsection{Upper Bound of Token Compression Loss}

We first study the compression error at each layer and then the compression error at the last layer, i.e., the token compression loss.

To measure the information loss during token compression, we study \textbf{layer compression error}, i.e., the $L_1$ distance of the output in the $l$-th layer before and after compression $\mathcal{D}^{(l)}$:
\begin{equation}\label{eq:l1_dis}
    \mathcal{D}^{(l)} = ||\mathbf{y}^{(l)} - \hat{\mathbf{y}}^{(l)}||_1.
\end{equation}
The layer compression error in the 1st layer can be bounded by the following lemma:
\begin{lemma}\label{eq:lemma_l1_dis_layer1}
Given token compression choice $\mathbf{I}^{(1)}$, the 1st layer compression error $\mathcal{D}^{(1)}$ can be bounded by:
\begin{equation}
    \mathcal{D}^{(1)} \leq 2C^{(1)} - 2C^{(1)} \sum_{i=1}^n I_i^{(1)} A_i^{(1)},
\end{equation}
where $C^{(1)} = \|\textbf{V}^{(1)} \textbf{W}_O^{(1)}\|_\infty$ is a constant.
\end{lemma}

\begin{proof}
    By expanding 1st layer compression error Eqn.~(\ref{eq:l1_dis}) with Lemma~\ref{lemma:output_after_comp} and using $\hat{\mathbf{A}}^{(1)}=\mathbf{A}^{(1)}$, we can obtain:
    \begin{equation}
        \mathcal{D}^{(1)} = \\
        \left\|(\mathbf{1} - \frac{\mathbf{I}^{(1)}}{||\mathbf{A}^{(1)} \odot \mathbf{I}^{(1)}||_1}) \odot \mathbf{A}^{(1)} \textbf{V}^{(1)} \textbf{W}_O^{(1)} \right\|_1,
    \end{equation}
    \begin{equation}
        \leq 
        \left\| (\mathbf{1} - \frac{\mathbf{I}^{(1)}}{||\mathbf{A}^{(1)} \odot \mathbf{I}^{(1)}||_1}) \odot \textbf{A}^{(l)} \right\|_1 
        \left\|\textbf{V}^{(1)} \textbf{W}_O^{(1)}\right\|_\infty,
    \end{equation}
    \begin{equation}
        \leq C^{(1)}  \left\| (\mathbf{1} - \frac{\mathbf{I}^{(1)}}{||\mathbf{A}^{(1)} \odot \mathbf{I}^{(1)}||_1}) \odot \textbf{A}^{(l)} \right\|_1.
    \end{equation}
    Denote $ C^{(1)} = ||\textbf{V}^{(l)} \textbf{W}_O^{(l)}||_\infty$, and we apply Hölder's inequality to derive the first inequality.
    
    Denote $F^{(1)} = ||\mathbf{A}^{(1)} \odot \mathbf{I}^{(1)}||_1 \in (0, 1] $, and consider the definition of $I_i^{(l)}$ in Eqn.~(\ref{eq:definition_I}), we can further simplify the compression error by expanding $\mathbf{A}^{(1)}$:
    \begin{equation}
        \mathcal{D}^{(1)} 
        \leq 
        C^{(1)}  \left\| (\mathbf{1} - \frac{\mathbf{I}^{(1)}}{||\mathbf{A}^{(1)} \odot \mathbf{I}^{(1)}||_1}) \odot \textbf{A}^{(l)} \right\|_1,
    \end{equation}
    \begin{equation}
        = C^{(1)} \sum_{i=1}^n \frac{|F^{(1)} - I_i^{(1)}| A_i^{(1)}}{F^{(1)}},
    \end{equation}
    \begin{equation}
        = C^{(1)} \sum_{\substack{i=1 \\ I_i^{(1)}=0}}^n A_i^{(1)} + C^{(1)} \sum_{\substack{i=1 \\ I_i^{(1)}=1}}^n \frac{(1 - F^{(1)}) A_i^{(1)}}{F^{(1)}},
    \end{equation}
    \begin{equation}
        = C^{(1)} \sum_{\substack{i=1 \\ I_i^{(1)}=0}}^n A_i^{(1)} + 
        C^{(1)} \left( \frac{\sum_{\substack{i=1 \\ I_i^{(1)}=1}}^n A_i^{(1)}}{F^{(1)}} - \sum_{\substack{i=1 \\ I_i^{(1)}=1}}^n A_i^{(1)} \right).
    \end{equation}
    
    Considering that there exists the following relationship between $I_i^{(1)}$ and $A_i^{(1)}$:
    \begin{equation}
        F^{(1)} = \sum_{i=1}^n A_i^{(1)}I_i^{(1)} = \sum_{\substack{i=1 \\ I_i^{(1)}=1}}^n A_i^{(1)}.
    \end{equation}
    The compression error can be further simplified:
    \begin{equation}
        \mathcal{D}^{(1)} 
        \leq 
        C^{(1)} \sum_{\substack{i=1 \\ I_i^{(1)}=0}}^n A_i^{(1)} + 
        C^{(1)} \left( 1 - \sum_{\substack{i=1 \\ I_i^{(1)}=1}}^n A_i^{(1)} \right),
    \end{equation}    
    \begin{equation}
        = 2C^{(1)} \sum_{\substack{i=1 \\ I_i^{(1)}=0}}^n A_i^{(1)},
    \end{equation}
    \begin{equation}
        = 2C^{(1)} \sum_{i=1}^ n (1 - I_i^{(1)}) A_i^{(1)},
    \end{equation}
    \begin{equation}
        = 2C^{(1)} - 2C^{(1)} \sum_{j=1}^n I_i^{(1)} A_i^{(1)}.
    \end{equation}
\end{proof}

To figure out how layer compression error propagates through layers, we have the following lemma:
\begin{lemma}\label{lemma:error_propagation}
Let $C^{(l)} = \|\textbf{V}^{(l)} \textbf{W}_O^{(l)}\|_\infty$ and assume $4C^{(l)}>1$ for $l=1,2,\dots,L$. In the $l$-th layer, given token compression choice $\mathbf{I}^{(l)}$, the layer compression error $\mathcal{D}^{(l)}$ is bounded by $\epsilon^{(l)}$:
\begin{equation}
    \mathcal{D}^{(l)} \leq \epsilon^{(l)} = 2C^{(l)} - 2C^{(l)} \prod_{k=1}^{(l)} \sum_{i=1}^n I_i^{(k)} A_i^{(k)}.
\end{equation}
\end{lemma}

\begin{proof}

Following the same logic with proof of Lemma~\ref{eq:lemma_l1_dis_layer1}, the $l$-th layer compression error is bounded by:
\begin{equation}\label{eq:l1_dis_perturbed}
    \mathcal{D}^{(l)} \leq 2C^{(l)} - 2C^{(l)} \sum_{i=1}^n I_i^{(l)} \hat{A}_i^{(l)}.
\end{equation}

Let $\delta\mathbf{q}^{(l)}\in\mathbb{R}^n$ be the perturbation of the input of the $l$-th attention layer after video compression:
\begin{equation}
    \hat{\mathbf{q}}^{(l)}=(\mathbf{q}^{(l)}+\delta\mathbf{q}^{(l)}).
\end{equation}
Considering that $\hat{\mathbf{q}}^{(l)}$ is obtained by applying an MLP followed by a linear layer to $\hat{\mathbf{y}}^{(l-1)}$ and $||\hat{\mathbf{y}}^{(l-1)}-\mathbf{y}^{(l-1)}||_1\le\epsilon^{(l-1)}$, the perturbation $\delta\mathbf{q}^{(l)}$ is also bound by:
\begin{equation}\label{eq:q_bound}
    ||\delta\mathbf{q}^{(l)}||_1 \le \epsilon^{(l-1)}.
\end{equation}

Then, $\hat{\textbf{A}}^{(l)}$ can be calculated with:
\begin{equation}
    \hat{\textbf{A}}^{(l)}=\hat{\mathbf{q}}^{(l)}(\mathbf{K}^{(l)})^T=(\mathbf{q}^{(l)}+\delta\mathbf{q}^{(l)})(\mathbf{K}^{(l)})^T.
\end{equation}
The $i$-th element in $\hat{\textbf{A}}^{(l)}$ is:
\begin{equation}\label{eq:A_hat_softmax_expand}
    \hat{A}_i^{(l)} = \frac{\exp\left[
    \left(\mathbf{q}^{(l)} + \delta\mathbf{q}^{(l)} \right) \left(\mathbf{K}_i^{(l)}\right)^T
    \right]}
    {\sum_{j=1}^{n} \exp\left[
        \left(\mathbf{q}^{(l)} + \delta\mathbf{q}^{(l)}\right) \left(\mathbf{K}_j^{(l)}\right)^T
    \right]},
\end{equation}
where $\mathbf{K}_i^{(l)}$ is the $i$-th row vector in $\mathbf{K}^{(l)}$.

Using the Hölder’s inequality, we can bound the perturbation term:
\begin{equation}
    \left| \delta\mathbf{q}^{(l)} \left(\mathbf{K}_j^{(l)}\right)^T
    \right| \le
    ||\delta\mathbf{q}^{(l)}||_1
    ||\mathbf{K}_j^{(l)}||_{\infty}.
\end{equation}
Based on Eqn.~(\ref{eq:q_bound}), we obtain:
\begin{equation}
\left\{
\begin{aligned}
    \delta \mathbf{q}^{(l)} \left( \mathbf{K}_j^{(l)} \right)^T &\ge - \epsilon^{(l-1)} ||\mathbf{K}_j^{(l)}||_{\infty}, \\
    \delta \mathbf{q}^{(l)} \left( \mathbf{K}_j^{(l)} \right)^T &\le \epsilon^{(l-1)} ||\mathbf{K}_j^{(l)}||_{\infty}.
\end{aligned}
\right.
\end{equation}
And the exponent in Eqn.~(\ref{eq:A_hat_softmax_expand}) satisfies:
\begin{equation}
    \begin{aligned}
    \mathbf{q}^{(l)} \left( \mathbf{K}_j^{(l)} \right)^T 
    - 
    \epsilon^{(l-1)} ||\mathbf{K}_j^{(l)}||_{\infty}
    \le \\
    \left(\mathbf{q}^{(l)} + \delta\mathbf{q}^{(l)} \right)
    \left( \mathbf{K}_j^{(l)} \right)^T
    \le \\
    \mathbf{q}^{(l)} \left( \mathbf{K}_j^{(l)} \right)^T 
    + 
    \epsilon^{(l-1)} ||\mathbf{K}_j^{(l)}||_{\infty}.
    \end{aligned}
\end{equation}

Therefore, the $i$-th element in Eqn.~(\ref{eq:A_hat_softmax_expand}) is bounded by:
\begin{equation}
    \hat{A}_i^{(l)} = \frac{\exp\left[
    \left(\mathbf{q}^{(l)} + \delta\mathbf{q}^{(l)} \right) \left(\mathbf{K}_i^{(l)}\right)^T
    \right]}
    {\sum_{j=1}^{n} \exp\left[
        \left(\mathbf{q}^{(l)} + \delta\mathbf{q}^{(l)}\right) \left(\mathbf{K}_j^{(l)}\right)^T
    \right]},
\end{equation}
\begin{equation}
    \ge
    \frac{\exp\left[
    \mathbf{q}^{(l)} \left( \mathbf{K}_i^{(l)} \right)^T 
    - 
    \epsilon^{(l-1)} ||\mathbf{K}_i^{(l)}||_{\infty}
    \right]}
    {\sum_{j=1}^{n} \exp\left[
        \mathbf{q}^{(l)} \left( \mathbf{K}_j^{(l)} \right)^T 
    + 
    \epsilon^{(l-1)} ||\mathbf{K}_j^{(l)}||_{\infty}
    \right]},
\end{equation}
\begin{equation}
    \ge
    A_i^{(l)}
    \frac{\exp\left[
    -\epsilon^{(l-1)} \max\left\{||\mathbf{K}_i^{(l)}||_{\infty}
    \right\}\right]}
    {\exp\left[
    \epsilon^{(l-1)} \max\left\{||\mathbf{K}_i^{(l)}||_{\infty}
    \right\}
    \right]},
\end{equation}
\begin{equation}
    =
    A_i^{(l)}
    \exp\left[
    -2\epsilon^{(l-1)} \max\left\{||\mathbf{K}_i^{(l)}||_{\infty}
    \right\}\right],
\end{equation}
\begin{equation}
    \ge
    A_i^{(l)}
    \exp\left[
    -2\epsilon^{(l-1)}
    \right],
\end{equation}
\begin{equation}\label{eq:A_perturbed_bound}
    \ge
    A_i^{(l)}
    \left[
    -2\epsilon^{(l-1)} + 1
    \right].
\end{equation}

We do not know the exact form of $\epsilon^{(l-1)}$ by now. But we can use mathematical induction to prove it. The base case $l-1=1$ has been proved in Lemma~\ref{eq:lemma_l1_dis_layer1}. We hypothesize that the following inductive hypothesis holds:
\begin{equation}\label{eq:induc_hypo}
    \epsilon^{(l-1)} = 2C^{(l-1)} - 2C^{(l-1)} \prod_{k=1}^{(l-1)} \sum_{i=1}^n I_i^{(k)} A_i^{(k)}.
\end{equation}
Combine Eqn.~(\ref{eq:l1_dis_perturbed}), Eqn.~(\ref{eq:A_perturbed_bound}), and assumption $4C^{(l-1)}>1$, we can prove the inductive step:
\begin{equation}
    \mathcal{D}^{(l)} \leq 2C^{(l)} - 2C^{(l)} \sum_{i=1}^n I_i^{(l)} \hat{A}_i^{(l)},
\end{equation}
\begin{equation}
    \leq 2C^{(l)} - 2C^{(l)} \left[
    -2\epsilon^{(l-1)} + 1
    \right] \sum_{i=1}^n I_i^{(l)} A_i^{(l)},
\end{equation}
\begin{equation}
    \begin{aligned}
    &= 2C^{(l)} - \\
    &2C^{(l)} \left[
    % -2\epsilon^{(l-1)} + 1
    -4C^{(l-1)} + 4C^{(l-1)} \prod_{k=1}^{(l-1)} \sum_{i=1}^n I_i^{(k)} A_i^{(k)} + 1
    \right] \\
    &\sum_{i=1}^n I_i^{(l)} A_i^{(l)},
    \end{aligned}
\end{equation}
\begin{equation}
    \leq 2C^{(l)} -
    2C^{(l)} \left[
    \prod_{k=1}^{(l-1)} \sum_{i=1}^n I_i^{(k)} A_i^{(k)}
    \right]
    \sum_{i=1}^n I_i^{(l)} A_i^{(l)},
\end{equation}
\begin{equation}
    = 2C^{(l)} -
    2C^{(l)} \left[
    \prod_{k=1}^{(l)} \sum_{i=1}^n I_i^{(k)} A_i^{(k)}
    \right].
\end{equation}
Finally,
\begin{equation}
    \mathcal{D}^{(l)} \leq 2C^{(l)} - 2C^{(l)} \prod_{k=1}^{(l)} \sum_{i=1}^n I_i^{(k)} A_i^{(k)} = \epsilon^{(l)}.    
\end{equation}
By the principle of mathematical induction, since the statement is true for the base case and the inductive step ensures it holds for all subsequent values, the statement is proven for all $l$ in the domain.
\end{proof}

To measure the final information loss during token compression, we define the \textbf{compression loss} $\mathcal{L}$ as the $L_1$ distance of the output in the last (i.e., $L$-th) layer before and after compression:
\begin{equation}
    \mathcal{L} = ||\mathbf{y}^{(l)} - \hat{\mathbf{y}}^{(l)}||_1.
\end{equation}
By applying Lemma~\ref{lemma:error_propagation} for $l=L$. The following theorem can characterize the upper bound of the compression loss.

\begin{theorem}\label{eq:token_compression_loss_bound}
Let $C^{(l)} = \|\textbf{V}^{(l)} \textbf{W}_O^{(l)}\|_\infty$ and assume $4C^{(l)}>1$ for $l=1,2,\dots,L$.
Given token compression choices $\left\{\mathbf{I}^{(l)}\right\}_{l=1}^{(L)}$, the compression loss $\mathcal{L}$ can be bounded by $\epsilon^L$:
\begin{equation}
    \mathcal{L} 
    \leq 
    \epsilon^L
    = 
    2C^{(L)} - 2C^{(L)} \prod_{l=1}^{L} \sum_{i=1}^n I_i^{(l)} A_i^{(l)}.
\end{equation}
\end{theorem}

\subsection{Minimizing the Upper Bound of Token Compression Loss}

We can prove that by selecting tokens based on top $C_{max}$ values in $\left\{A_i^l\right\}$, we can achieve a near near-optimal minimization of the upper bound of token compression loss in Lemma~\ref{eq:token_compression_loss_bound}.

\begin{theorem}
Given the token sequence budget $\sum_l\sum_iI_i^{(l)}=K$, selecting token compression choices $\left\{\mathbf{I}_*^{(l)}\right\}_{l=1}^{(L)}$ based on top $K$ values in $\left\{A_i^l\right\}$ can achieve a near-optimal minimization of the upper bound of token compression loss to $\epsilon_*^L$:
\begin{equation}
    \epsilon_*^L
    \le
    2C^{(L)}
    +
    2C^{(L)}
    \left(\frac{\epsilon_{opt}^L}{2C^{(L)}}-1\right)^
    {1-\frac{1}{e}},
\end{equation}
where $\epsilon_{opt}^L$ is the theoretical minimal of $\epsilon^L$.
\end{theorem}

\begin{proof}
% 1.Reformulating the Product Objective as a Log-Sum
To minimize $\epsilon^L$, we can maximize the product-sum within:
\begin{equation}
    F(\mathcal{S}) = \prod_{l=1}^L \sum_{i \in \mathcal{S}_l} A^{(l)}_i = \exp\left(\sum_{l=1}^L \log\left(\sum_{i \in \mathcal{S}_l} A^{(l)}_i\right)\right),
\end{equation}
where \(\mathcal{S} = \{\mathcal{S}_1, \mathcal{S}_2, \dots, \mathcal{S}_L\}\) denotes the sets of selected elements for each layer, with \(\sum_{l=1}^L |\mathcal{S}_l| = K\).

Maximizing \(F(\mathcal{S})\) is equivalent to maximizing the log-product $f(\mathcal{S})$:
\begin{equation}
    f(\mathcal{S}) = \sum_{l=1}^L \log\left(\sum_{i \in \mathcal{S}_l} A^{(l)}_i\right).
\end{equation}

% . Submodularity of \(f(\mathcal{S})\)
Next, we prove its submodular property for minimization. A set function \(f\) is \textit{submodular} if it satisfies diminishing marginal returns:
\begin{equation}
    \begin{aligned}
        & f(\mathcal{S} \cup \{e\}) - f(\mathcal{S}) \ge \\ & f(\mathcal{T} \cup \{e\}) - f(\mathcal{T}), \quad \forall \mathcal{S} \subseteq \mathcal{T}.
    \end{aligned}
\end{equation}

For our problem, adding an element \(e\) (from layer \(l\)) to \(\mathcal{S}\) increases \(f(\mathcal{S})\) by:
\begin{equation}
    \Delta_f(e \mid \mathcal{S}) = \log\left(S_l + A^{(l)}_e\right) - \log(S_l),
\end{equation}
where \(S_l = \sum_{i \in \mathcal{S}_l} A^{(l)}_i\). Since \(\Delta_f(e \mid \mathcal{S})\) decreases as \(S_l\) grows (diminishing returns), \(f\) is \textit{monotone submodular}.

% Greedy Algorithm Guarantees
Selecting top$K$ values in $\left\{A_i^l\right\}$ is equal to
the greedy algorithm that iteratively selects the element \(e\) that maximizes \(\Delta_f(e \mid \mathcal{S})\). For monotone submodular functions under a cardinality constraint \(K\), the greedy algorithm achieves a $(1 - 1/e)$-approximation of the optimal solution \cite{nemhauser1978_submodular_analysis}. Formally, let \(f(\mathcal{S}^*)\) be the optimal value and \(f(\mathcal{S}_{\text{greedy}})\) the greedy solution. From submodularity theory:
\[
f(\mathcal{S}_{\text{greedy}}) \ge \left(1 - \frac{1}{e}\right) f(\mathcal{S}^*).
\]
For multiplicative objectives like \(F(\mathcal{S}) = e^{f(\mathcal{S})}\), this translates to:
\[
F(\mathcal{S}_{\text{greedy}}) \ge e^{(1 - 1/e) \log F(\mathcal{S}^*)} = F(\mathcal{S}^*)^{1 - 1/e}.
\]

Therefore,
\begin{equation}
    \prod_{l=1}^{L} \sum_{i=1}^n {I_*}_i^{(l)} {A_*}_i^{(l)}
    \ge
    \left(\prod_{l=1}^{L} \sum_{i=1}^n {I_{opt}}_i^{(l)} {A_{opt}}_i^{(l)}\right)^{1-\frac{1}{e}}.
\end{equation}

Based on Eqn.~(\ref{eq:token_compression_loss_bound}), we can finally get:
\begin{equation}
    \epsilon_*^L
    \le
    2C^{(L)}
    +
    2C^{(L)}
    \left(\frac{\epsilon_{opt}^L}{2C^{(L)}}-1\right)^
    {1-\frac{1}{e}}.
\end{equation}

% **References**  
% [1] Nemhauser, G. L., Wolsey, L. A., & Fisher, M. L. (1978). An analysis of approximations for maximizing submodular set functions. *Mathematical Programming*.

\end{proof}

\end{document}